\documentclass{article}

\usepackage{arxiv}

\usepackage[utf8]{inputenc} 
\usepackage[T1]{fontenc}    
\usepackage{hyperref}       
\usepackage{url}            
\usepackage{booktabs}       
\usepackage{amsfonts}       
\usepackage{nicefrac}       
\usepackage{microtype}      
\usepackage{lipsum}		
\usepackage{subcaption}
\usepackage{graphicx}
\usepackage{natbib}
\usepackage{doi}
\usepackage{amsmath}
\usepackage{amsthm}
\usepackage{xcolor}
\usepackage{comment}
\usepackage{amsthm}
\usepackage{amssymb}
\usepackage[normalem]{ulem}

\newtheorem*{rep@theorem}{\rep@title}
\newcommand{\newreptheorem}[2]{%
\newenvironment{rep#1}[1]{%
 \def\rep@title{#2 \ref{##1}}%
 \begin{rep@theorem}}%
 {\end{rep@theorem}}}
\makeatother

\newtheorem{lemma}{Lemma}
\newtheorem{theorem}{Theorem}

\newreptheorem{theorem}{Theorem}
\newreptheorem{proposition}{Proposition}
\newreptheorem{corollary}{Corollary}
\newreptheorem{lemma}{Lemma}
\numberwithin{theorem}{section}

\newcommand{\R}{\mathbb{R}}
\newcommand{\dd}{\mathrm{d}}

\title{Latent Nonlinear Denoising Score Matching for Enhanced Learning of Structured Distributions}

\date{} 					

\author{ \href{}{\hspace{1mm}Kaichen Shen} \\
	School of Mathematics\\
	Georgia Institute of Technology\\
	Atlanta, GA\\
	\texttt{kshen77@gatech.edu} \\
	\And
	\href{}{\hspace{1mm}Wei Zhu} \\
	School of Mathematics\\
	Georgia Institute of Technology\\
	Atlanta, GA\\
	\texttt{weizhu@gatech.edu} \\
}

\renewcommand{\shorttitle}{}

\hypersetup{
pdftitle={Latent nonlinear denoising score matching for enhanced learning of structured distributions},
pdfsubject={q-bio.NC, q-bio.QM},
pdfauthor={Kaichen Shen, Wei Zhu},
pdfkeywords={Score-based generative modeling; structure-preserving generative modeling; denoising score-matching; latent space; learning from scarce data},
}

\begin{document}
\maketitle

\begin{abstract}


    Score-based generative models commonly employ linear forward diffusions with Gaussian reference distributions, limiting their ability to directly incorporate structural information into the noising dynamics. We introduce latent nonlinear denoising score matching (LNDSM), which combines nonlinear, structure-informed latent diffusions with the variational autoencoder (VAE) training framework. In our construction, the forward diffusion is designed to admit a prescribed reference distribution, such as a Gaussian mixture model, that reflects qualitative features of the latent data distribution. We derive a computable approximation to a cross-entropy upper bound in the VAE objective using Euler–Maruyama discretization and the resulting Gaussian conditional transitions. To stabilize training, we remove two zero-mean control variates whose variances become large for small time steps. The resulting objective enables joint training of the encoder, decoder, and latent score network. Numerical experiments on benchmark image datasets demonstrate the effectiveness of LNDSM in learning structured distributions, outperforming conventional latent score-based generative modeling in sample quality and mode balance across several settings, while achieving improved training efficiency. These results highlight the potential of nonlinear, structure-informed latent diffusions for generative modeling.
\end{abstract}



\section{Introduction}

Score-based generative models (SGMs) provide a flexible framework for learning complex data distributions by reversing a prescribed diffusion process that connects the data distribution to a tractable reference distribution~\cite{DDPM,song2021scorebased}. Standard implementations commonly employ linear forward diffusions with Gaussian reference distributions, for which the transition densities and their conditional scores are available in closed form. This property makes the denoising score matching (DSM) objective tractable, but limits the flexibility of the forward dynamics in incorporating prior structural information.

In many applications, target distributions exhibit structural properties, such as multimodality, group symmetries, or concentration near lower-dimensional manifolds. Incorporating such information can be particularly valuable when training data are limited. To this end, Birrell et al.~\cite{NDSM} introduced nonlinear denoising score matching (NDSM), which incorporates a relaxed structural bias by choosing a reference distribution that shares qualitative features with the target. In particular, a nonlinear Langevin diffusion can be constructed to admit the prescribed reference distribution as its invariant measure. Unlike standard Gaussian noising processes, this construction allows structural information to enter directly into the forward dynamics. Since the transition densities of nonlinear diffusions are generally unavailable in closed form, NDSM uses the Euler--Maruyama (EM) discretization to obtain tractable Gaussian conditional transitions for score matching.

A complementary direction is to perform generative modeling in a learned latent space rather than directly in the data space~\cite{LSGM,rombach2022highresolutionimagesynthesislatent,pinaya2022brain,blattmann2023align,xu2023geometric}. Latent representations can reduce the computational dimension and reveal structural features that are less apparent in the original data space. A representative example is latent score-based generative modeling (LSGM)~\cite{LSGM}, which combines a variational autoencoder (VAE) with an SGM that learns the prior distribution over latent variables. The encoder, decoder, and latent score network are trained jointly through a VAE objective involving the cross entropy between the encoder distribution and the learned latent prior. However, the tractable formulation of LSGM relies on a linear latent diffusion with Gaussian conditional transitions, limiting its ability to directly incorporate structure-informed nonlinear dynamics. Although NDSM has also been applied in a fixed latent space, its score-matching objective does not directly provide a tractable cross-entropy formulation for joint VAE training with a learnable encoder.

In this work, we introduce latent nonlinear denoising score matching (LNDSM), which combines the advantages of LSGM and NDSM by incorporating structure-informed nonlinear diffusions into joint VAE-based latent generative modeling. The main challenge is to obtain a tractable cross-entropy objective when the forward transition density is not available in closed form. To address this, we use the EM discretization of the nonlinear latent diffusion and exploit its Gaussian one-step conditional transitions to construct a computable approximation to the cross-entropy upper bound. This allows the intractable marginal score to be replaced by a tractable conditional score. Following the variance-reduction principle of NDSM, we further identify and remove two zero-mean terms that exhibit large variance for small EM time steps. The resulting LNDSM objective enables joint optimization of the VAE encoder, decoder, and latent score network while incorporating structural information through the nonlinear latent dynamics.

We demonstrate that our proposed LNDSM enhances the modeling of structured distributions and accelerates synthesis. Experiments on multimodal and approximately group symmetric dataset---especially in low-data regimes---show that LNDSM achieves lower Fr\'{e}chet inception distance (FID)~\cite{FID}, higher inception score (IS)~\cite{IS}, and significantly improved mode balance compared to LSGM, while requiring fewer training time. These results indicate that combining structured latent reference distributions with nonlinear latent diffusions provides a powerful and efficient framework for learning structured data distributions.

\subsection{Related Work}

\textbf{Structure-preserving generative models}. Structural information has been incorporated into generative models through specialized architectures that enforce symmetries or other constraints. Examples include group-equivariant GANs~\cite{deygroup,birrell2022structurepreservinggans} and equivariant diffusion models~\cite{hoogeboom2022equivariant,Diffusion-Models-under-Group-Transformations}. These approaches are particularly effective when the underlying symmetries are known and can be imposed explicitly. In contrast, NDSM~\cite{NDSM}, on which our work builds, incorporates a relaxed structural bias through the reference distribution, without requiring the target distribution to satisfy exact structural constraints.

\textbf{Flow matching and stochastic transport}. Flow matching provides an alternative approach to generative modeling by learning velocity fields associated with prescribed probability paths~\cite{lipman2023flow}. Generalized conditional flow matching~\cite{tong2024improving} accommodates arbitrary source distributions, while stochastic interpolants~\cite{albergo2025stochastic} and Schr\"{o}dinger bridges~\cite{de2021diffusion,pmlr-v238-tong24a} offer flexible deterministic or stochastic connections between distributions. In particular, Gaussian-mixture reference distributions have been explicitly investigated in conditional flow matching~\cite{josias2024multimodal}. Thus, the use of non-Gaussian reference distributions is not unique to our approach. The distinction is that LNDSM incorporates structural information through a prescribed nonlinear diffusion whose drift is determined by the reference distribution, rather than through the choice of an interpolation path. These approaches are complementary, and both provide ways to exploit prior information about the target distribution.

\textbf{Nonlinear and learnable diffusion dynamics}. Recent work has explored more flexible forward processes beyond conventional linear Gaussian diffusions. Neural Diffusion Models (NDMs)~\cite{bartosh2024neural} introduce learnable, time-dependent nonlinear transformations and have also been extended to diffusion in the latent space of a VAE. Neural Flow Diffusion Models (NFDMs)~\cite{bartosh2024neuralflow} further develop learnable forward processes through a variational training framework that avoids trajectory simulation during training. These methods demonstrate that nonlinear latent diffusion can be incorporated into variational generative modeling. Our approach differs in its objective: rather than learning the forward dynamics, LNDSM prescribes a nonlinear diffusion whose invariant distribution encodes structural information and develops an EM-based, variance-reduced training objective for the resulting latent score model.

\section{Background and Motivation}\label{section:2}


\subsection{Score-based Generative Models (SGMs)}\label{subsection: score-based generative modeling}


We first review score-based generative modeling through stochastic differential equations (SDEs) \cite{song2021scorebased}. Let $p_0 = p_{\text{data}}$ be a target density on $\R^d$, and consider the forward diffusion
\begin{align}
    \label{eq:general forward SDE}
    & \dd x_{t} = f(x_{t}, t) \dd t + g(t) \dd w_{t}, \quad x_0\sim p_0, \quad t\in [0, T].
\end{align}
Here $f:\mathbb{R}^{d}\times [0, T]\rightarrow\mathbb{R}^{d}$ is the drift, $g(\cdot):[0, T]\rightarrow [0, \infty)$ is a scalar diffusion coefficient, and $w_{t}$ is a standard $d$-dimensional Wiener process. Denote the marginal density of $x_t$ by $p_t$. The forward coefficients and terminal time $T$ are chosen so that the marginal density $p_T$ is close to a reference density $\pi$ from which samples can be drawn efficiently.

Under suitable regularity conditions, the time-reversal result of Anderson~\cite{anderson1982reverse} gives the reverse-time SDE
\begin{align}
    \label{eq:general reverse-time SDE}
    & \dd x_{t} = \left[ f(x_{t}, t) - g^{2}(t)\nabla_{x_{t}}\log p_{t}(x_{t}) \right] \dd t + g(t) \dd \bar{w}_{t},
\end{align}
where time runs from $T$ to $0$ and $\bar w_t$ denotes a reverse-time
Wiener process. When initialized from the exact terminal density $p_T$ and supplied with the exact \textit{marginal score} $\nabla_{x_t}\log p_t(x_t)$, this reverse process recovers the data
distribution $p_0$ at time zero~\cite{song2021scorebased}.

In practice, two approximations are introduced: the terminal density $p_T$ is replaced by the reference density $\pi$, and the unknown marginal score
$\nabla_{x_t}\log p_t(x_t)$ is approximated by a time-dependent neural network $s_\theta(x_t,t)$. The latter is typically learned via denoising score matching (DSM)~\cite{song2021scorebased,DSM}. Specifically,
letting $p_{0t}(x_t\mid x_0)$ denote the transition density associated with
\eqref{eq:general forward SDE}, the network $s_\theta$ is trained by
\begin{equation}
    \label{eq:denoising score matching}
    \theta^*=\operatorname*{arg\,min}_{\theta}
  \mathbb{E}_{t\sim\operatorname{Unif}(0,T)}
  \left[
    \lambda(t)\,
    \mathbb{E}_{x_0\sim p_0}
    \mathbb{E}_{x_t\sim p_{0t}(\cdot\mid x_0)}
    \left\|
      s_\theta(x_t,t)-\nabla_{x_t}\log p_{0t}(x_t\mid x_0)
    \right\|_2^2
  \right],
\end{equation}
where $\lambda(t)>0$ is a weighting function. For a fixed data distribution and forward process, this objective differs from (explicit) marginal score matching only by a term independent of $\theta$, and therefore has the same minimizers. The transition density
$p_{0t}(x_t\mid x_0)$ admits a closed-form expression when the forward
diffusion~\eqref{eq:general forward SDE} is an \textit{Ornstein--Uhlenbeck
(OU) process}, namely, when the drift is linear in $x_t$ and takes the form
$f(x_t,t)=-\alpha(t)x_t$. In this case, the transition density is Gaussian,
with explicitly known mean and covariance, making the DSM objective tractable. This linear OU structure therefore plays a
central role in the practical implementation of classical DSM for score-based generative models.

Once training is complete, the intractable marginal score
$\nabla_{x_t}\log p_t(x_t)$ in the reverse-time SDE is replaced by its
learned approximation $s_{\theta^*}(x_t,t)$. Numerically simulating the
resulting reverse-time SDE then yields samples from the learned
approximation to the target distribution $p_{\text{data}}$.


\subsection{Latent Score-based Generative Models (LSGMs)}\label{subsection: LSGM}

Latent score-based generative modeling (LSGM)~\cite{LSGM} combines a
variational autoencoder (VAE) with a score-based generative model
in latent space. The encoder $q_\phi(z_0\mid x)$ maps a data point
$x$ to a distribution over latent variables
$z_0\in\mathbb{R}^{d_z}$, while the decoder $p_\psi(x\mid z_0)$
maps latent variables back to a distribution over data. Unlike a
standard VAE, which typically uses a fixed Gaussian prior $p(z_0)= \mathcal{N}(0, I)$, LSGM
learns the latent prior through a diffusion model. Specifically,
a score network $s_\theta(z_t,t)$ parameterizes a reverse-time SDE
initialized from a Gaussian reference distribution $\pi$ at time
$T$. The distribution produced at time zero defines the learned
prior $p_\theta(z_0)$.

The encoder, decoder, and score network are trained jointly
within the VAE framework. For a data point $x$, the negative
evidence lower bound is
\begin{align}
    \nonumber
    \mathcal{L}(x, \phi, \theta, \psi) &= \mathbb{E}_{q_{\phi}(z_{0}|x)}[-\log p_{\psi}(x|z_{0})] + D_{\text{KL}}(q_{\phi}(z_{0}|x) \rVert p_{\theta}(z_{0})) \\
    \label{eg:one-point loss LSGM}
    &= \mathbb{E}_{q_{\phi}(z_{0}|x)}[-\log p_{\psi}(x|z_{0})] + \mathbb{E}_{q_{\phi}(z_{0}|x)}[\log q_{\phi}(z_{0}|x)] + \text{CE}(q_{\phi}(z_{0}|x) \rVert p_{\theta}(z_{0})).
\end{align}
The three terms in~\eqref{eg:one-point loss LSGM} correspond to the reconstruction loss, the negative encoder entropy, and the cross entropy between $q_{\phi}(z_0 |x)$  and the learned latent prior $p_{\theta}(z_0)$. Training minimizes the expectation of this objective over
$x\sim p_{\mathrm{data}}$.

The main difficulty lies in the cross-entropy term, since the density of the SGM prior is not explicitly available. For the linear latent diffusion used in LSGM, the transition density $p_{0t}(z_t\mid z_0)$ is Gaussian and its conditional score is available in closed form. Under the regularity assumptions and terminal-matching assumption of \cite[Theorem~1]{LSGM}, together with the likelihood bound established in~\cite{song2021maximum}, the cross entropy for the learned reverse-SDE prior satisfies
\begin{align}
    \nonumber
    \text{CE}(q_{\phi}(z_{0}|x) \rVert p_{\theta}(z_{0})) &= \mathbb{E}_{q_{\phi}(z_{0}|x)}[-\log p_{\theta}(z_{0})] \\
    \label{eq:cross entropy}
    &\le T\, \mathbb{E}_{t}\left[\frac{g^{2}(t)}{2}\mathbb{E}_{q_{\phi}(z_{t}, z_{0}|x)}\left[\rVert s_{\theta}(z_{t}, t) - \nabla_{z_{t}}\log p_{0t}(z_{t}|z_{0}) \rVert_{2}^{2}\right]\right] + c,
\end{align}
where $t$ is uniformly sampled from $[0, T]$, $q_{\phi}(z_{t}, z_{0}|x) = p_{0t}(z_{t}|z_{0})q_{\phi}(z_{0}|x)$, and $c$ is a constant independent of $\phi$, $\theta$, and $\psi$. 
Theorem~1 of~\cite{LSGM} establishes an exact cross-entropy identity when the score is the exact forward marginal score of the prior. For a learned score network, however, the distribution defined by the reverse-time SDE need not have the network as its forward marginal score. In this case, the corresponding expression provides an upper bound on the cross entropy, as clarified immediately after that theorem and discussed in detail by Song et al.~\cite{song2021maximum}. Replacing the cross-entropy term in \eqref{eg:one-point loss LSGM} by the right-hand side of \eqref{eq:cross entropy} therefore yields a tractable upper bound on the negative evidence lower bound. The resulting objective allows the encoder, decoder, and latent score network to be trained jointly. Its tractability relies on the available Gaussian transition score of the linear diffusion; for generic nonlinear latent dynamics, this transition score is not available in closed form.

\subsection{Nonlinear Denoising Score Matching (NDSM)}\label{subsection:NDSM}

A complementary approach to improving SGMs is to incorporate
structural information through the forward diffusion. In this
subsection, we review the nonlinear denoising score-matching
framework proposed in~\cite{NDSM}, which we refer to as NDSM-SGM. The method selects a reference density $\pi(x)$ that shares
qualitative structural features with the target density
$p_0(x)$, such as multimodality or approximate symmetries.
For a smooth, positive reference density, a natural choice of
forward process is the Langevin diffusion
\begin{align}
    \label{eq:nonlinear forward SDE text}
    \dd x_{t} = -\nabla_{x_{t}}V(x_{t}) \dd t + \sqrt{2} \dd w_{t},\quad V(x) := -\log(\pi(x)),
\end{align}
Under suitable regularity and ergodicity assumptions, this
diffusion has invariant density $\pi$ and converges to the
corresponding reference distribution. Importantly, the reference
distribution is \textit{not required to share exactly the same
structure} as the target; instead, it provides a relaxed
structural bias through the forward dynamics.

For a generic nonlinear drift $-\nabla_{x_t}V(x_t)$, the
transition density $p_{0t}(x_t\mid x_0)$ is not available in
closed form, preventing direct evaluation of the classical
DSM objective. NDSM addresses this by
applying the Euler--Maruyama (EM) discretization to a general
forward SDE, yielding a discrete-time Markov chain
$\{x_n\}_{n=0}^{n_f}$ that approximates the continuous diffusion.
With $0=t_0<t_1<\cdots<t_{n_f}=T$, the EM chain is defined by
\begin{align}
    \label{eq:EM Markov process}
    x_{0} &\sim p_{0}(x), \\
    x_{n} &= x_{n-1} + f(x_{n-1}, t_{n-1})\Delta t_{n-1} + g(t_{n-1})\sqrt{\Delta t_{n-1}}U_{n-1} \nonumber \\
    &:= \mu(x_{n-1}, t_{n-1}, \Delta t_{n-1}) + \sigma(x_{n-1}, t_{n-1}, \Delta t_{n-1})U_{n-1},\quad 1 \leq n \leq n_{f},
\end{align}
where $\Delta t_{n-1} = t_{n} - t_{n-1}$ and
$U_0,\ldots,U_{n_f-1}$ are independent standard Gaussian
vectors, independent of $x_0$. The conditional transition
density of the EM chain is exactly
\begin{align}
    \label{eq:EM Markove process transition}
    p(x_n|x_{n-1})= \mathcal{N}\left(x_n;\mu(x_{n-1}, t_{n-1}, \Delta t_{n-1}), \sigma^{2}(x_{n-1}, t_{n-1}, \Delta t_{n-1})I\right).
\end{align}
Although sufficiently small time steps are needed to accurately approximate the continuous diffusion, the Gaussianity of the discrete conditional transitions holds by construction.

For the resulting Markov chain, NDSM replaces direct minimization of the intractable expected squared error between the score network and the marginal score with the following equivalent, computable objective:
\begin{align}
    \label{eq:NDSM}
    &\operatorname*{arg\,min}_{\theta}\frac{1}{2}\mathbb{E}_{N}\left[\rVert s_{\theta}(x_{N}, t_{N}) - \nabla_{x_{N}}\log p_{N}(x_{N}) \rVert_{2}^{2}\right] = \operatorname*{arg\,min}_{\theta}\mathbb{E}\left[\mathcal{L}_{\theta, N}^{\text{NDSM}}\right],
\end{align}
where
\begin{align*}
    \mathcal{L}_{\theta, N}^{\text{NDSM}} := \frac{1}{2}\rVert s_{\theta}(x_{N}, t_{N}) \rVert_{2}^{2} + \frac{U_{N-1}\cdot\left[s_{\theta}(x_{N}, t_{N}) - s_{\theta}(\mu(x_{N-1}, t_{N-1}, \Delta t_{N-1}), t_{N})\right]}{\sigma(x_{N-1}, t_{N-1}, \Delta t_{N-1})}.
\end{align*}
Here $p_n$ denotes the marginal density of $x_n$, and $N$ is
a random time index uniformly distributed over
$\{1,\ldots,n_f\}$, independent of $x_0$ and the Gaussian
noises $\{U_n\}$. The expectations are taken over the random
time index and the EM trajectories. This construction yields
a tractable training objective for SGMs with nonlinear forward
dynamics while incorporating structural information through
the choice of the reference distribution $\pi(x)$.



NDSM-SGM~\cite{NDSM} provides a general framework for incorporating
structural information through nonlinear diffusion dynamics. It has also been applied in latent space using a two-stage procedure: an autoencoder is first trained to obtain a latent representation of the data, and NDSM is subsequently trained in this fixed latent space. A limitation of this approach is that the latent representation is learned independently of the score-based generative objective and therefore is not adapted to the nonlinear diffusion model used for generation. In contrast, LSGM~\cite{LSGM} jointly trains the VAE and the latent score model, allowing the generative objective to influence the learned representation, but relies on a linear latent diffusion with a Gaussian reference distribution. These observations motivate our proposed framework, which combines joint VAE training with nonlinear, structure-informed latent diffusions.


\section{Nonlinear Denoising Score Matching in Latent Space}\label{section:3}

The goal of this work is to extend NDSM-SGM (section~\ref{subsection:NDSM}) to the latent space using the VAE-based training framework of LSGM (section~\ref{subsection: LSGM}). Following the same notation as LSGM, we denote the encoder by $q_{\phi}(z_{0}|x)$ and the decoder by $p_{\psi}(x|z_{0})$. We consider a general (nonlinear) diffusion process in the latent space of the form
\begin{align}
    \label{eq:generic diffusion process in the latent space}
    \dd z_{t} = f(z_{t}, t) \dd t + g(t) \dd w_{t}.
\end{align}

Using the same principle as in the NDSM framework, we impose a relaxed structural constraint on the latent diffusion by requiring it to converge to a prescribed reference distribution in the latent space. This reference distribution is designed to share \textit{similar, but not necessarily identical}, structural characteristics with the target \textit{latent} distribution induced by the encoder. We retain the VAE-based training objective, in analogy with~\eqref{eg:one-point loss LSGM}:
\begin{align}
    \label{eq:new VAE-induced loss function}
    \mathbb{E}_{x\sim p_{\text{\normalfont data}}}[\mathcal{L}(x, \phi, \theta, \psi)] =& \mathbb{E}_{x\sim p_{\text{\normalfont data}}}\left\{\mathbb{E}_{q_{\phi}(z_{0}|x)}\left[-\log p_{\psi}(x|z_{0})] + \mathbb{E}_{q_{\phi}(z_{0}|x)}[\log q_{\phi}(z_{0}|x)\right] + \mathbb{E}_{q_{\phi}(z_{0}|x)}\left[-\log p_{\theta}(z_{0})\right]\right\} \nonumber \\
    =& \mathbb{E}_{x\sim p_{\text{\normalfont data}}}\left\{\mathbb{E}_{q_{\phi}(z_{0}|x)}\left[-\log p_{\psi}(x|z_{0})\right] + \mathbb{E}_{q_{\phi}(z_{0}|x)}\left[\log q_{\phi}(z_{0}|x)\right] + \text{CE}(q_{\phi}(z_{0}|x) \rVert p_{\theta}(z_{0}))\right\}.
\end{align}
Here, the latent prior $p_{\theta}(z_{0})$ in the cross-entropy term is now learned via NDSM-SGM in the latent space, with $\theta$ being the parameter of the latent score network. The modeling of this cross entropy term is detailed in the following subsection.


\subsection{Cross Entropy Term}

The main theoretical and algorithmic contribution of this work is the integration of NDSM into the cross-entropy term of the VAE training objective~\eqref{eq:new VAE-induced loss function}. We first show (see Lemma~\ref{lemma:lemma1} in Appendix~\ref{section:appendix A} for details) that the cross-entropy term $\text{CE}(q_{\phi}(z_{0}|x) \rVert p_{\theta}(z_{0}))$ in the loss function~\eqref{eq:new VAE-induced loss function} can be upper bounded by
\begin{align}
\label{eq:cross_entropy_bound_cts}
    \text{CE}(q_{\phi}(z_{0}|x) \rVert p_{\theta}(z_{0})) \leq & H(q_{\phi}(z_{T}|x)) \nonumber +  \frac{1}{2}\int_{0}^{T}g^{2}(t)\mathbb{E}_{q_{\phi}(z_{t}|x)}\left[ \rVert \nabla_{z_{t}}\log q_{\phi}(z_{t} | x)  - s_{\theta}(z_{t}, t) \rVert_{2}^{2} \right] \dd t \nonumber \\
    & +\frac{1}{2}\int_{0}^{T}\mathbb{E}_{q_{\phi}(z_{t}|x)}\left[ \left[2f(z_{t}, t) - g^{2}(t)\nabla_{z_{t}}\log q_{\phi}(z_{t}|x) \right]^{T}\nabla_{z_{t}}\log q_{\phi}(z_{t} | x) \right] \dd t,
\end{align}
where the inequality arises from the use of a score network instead of the exact marginal score [see the comment after Eq.~\eqref{eq:cross entropy}], and  $H(q_{\phi}(z_{T}|x))$ is the entropy of the chosen latent reference distribution. Since the latent diffusion process~\eqref{eq:generic diffusion process in the latent space} has a nonlinear drift, the exact form of the density $q_{\phi}(z_{t} | x)$ is generally intractable. As in the NDSM framework in section~\ref{subsection:NDSM}, we therefore replace the intractable marginal density by a computable (local) conditional density obtained through time discretization. Specifically, we simulate the latent diffusion using the Euler–Maruyama (EM) scheme, which yields a discrete-time Markov process $\{z_{n}\}_{n=0}^{n_{f}}$, with $t_{0}=0$ and $t_{n_{f}}=T$:
\begin{align}
    \label{eq:EM Markov process 1}
    z_{0} &\sim q_{\phi}(z_{0}|x)p_{0}(x), \\
    z_{n} &= z_{n-1} + f(z_{n-1}, t_{n-1})\Delta t_{n-1} + g(t_{n-1})\sqrt{\Delta t_{n-1}}U_{n-1} \nonumber \\
    &:= \mu(z_{n-1}, t_{n-1}, \Delta t_{n-1}) + \sigma(z_{n-1}, t_{n-1}, \Delta t_{n-1})U_{n-1} \nonumber \\
    \label{eq:EM Markov process 2}
    &:= \mu_{n-1} + \sigma_{n-1}U_{n-1}, \quad 1 \leq n \leq n_{f},
\end{align}
where $\Delta t_{n-1} = t_{n} - t_{n-1}$ is sufficiently small, and $U_{n-1}\sim\mathcal{N}(0, I)$. Introducing a random time index $N$, uniformly distributed over $\{0, \cdots, n_{f}\}$, the cross-entropy term $\text{CE}(q_{\phi}(z_{0}|x) \rVert p_{\theta}(z_{0}))$ can be approximately upper bounded by:
\begin{align}
    \label{eq: main text cross entropy}
    \text{CE}(q_{\phi}(z_{0}|x) \rVert p_{\theta}(z_{0}))\lesssim\; & H(q_{\phi}(z_{n_{f}}))  +  \underbrace{\frac{{T}}{2}\mathbb{E}_{N, z_{N}}\left[ g^{2}(t_{N}) \rVert \nabla_{z_{N}}\log q_{\phi}(z_{N}) - s_{\theta}(z_{N}, t_{N}) \rVert_{2}^{2} \right]}_{\text{(I)}} \nonumber \\
    \quad + & \underbrace{\frac{{{T}}}{2}\mathbb{E}_{N, z_{N}}\left[ \left[ 2f(z_{N}, t_{N}) - g^{2}(t_{N})\nabla_{z_{N}}\log q_{\phi}(z_{N}) \right]^{T} \nabla_{z_{N}}\log q_{\phi}(z_{N}) \right]}_{\text{(II)}}
\end{align}
All continuous-time variables are thus replaced by their discrete-time counterparts.

The remaining task is to convert the marginal score $\nabla_{z_{N}}\log q_{\phi}(z_{N})$ in~\eqref{eq: main text cross entropy} into the conditional score $\nabla_{z_{N}}\log q_{\phi}(z_{N} | z_{N-1})$, which is tractable since the conditional distribution $z_{N} | z_{N-1}$ is Gaussian. Note that the squared-norm term (I) in~\eqref{eq: main text cross entropy} can be expanded: the squared norm of $s_{\theta}(z_{N}, t_{N})$ is computed by the score network directly, and the squared norm of $\nabla_{z_{N}}\log q_{\phi}(z_{N})$ is cancelled by the corresponding term in (II) of Eq.~\eqref{eq: main text cross entropy}. Therefore, the remaining terms containing the unknown marginal score $\nabla_{z_{N}}\log q_{\phi}(z_{N})$ are:

\noindent
\begin{align}
    \label{eq: remaining terms marginal score 1}
    &A =  \mathbb{E}_{N, z_{N}} \left[ g^{2}(t_{N}) s_{\theta}(z_{N}, t_{N}) \cdot \nabla_{z_{N}}\log q_{\phi}(z_{N}) \right], \\
    \label{eq: remaining terms marginal score 2}
    & B = \mathbb{E}_{N, z_{N}} \left[ f(z_{N}, t_{N}) \cdot \nabla_{z_{N}}\log q_{\phi}(z_{N}) \right].
\end{align}
Using a derivation similar to those in~\cite{DSM,song2021scorebased,DDPM}, the marginal score $\nabla_{z_{N}}\log q_{\phi}(z_{N})$ appearing in~\eqref{eq: remaining terms marginal score 1} and~\eqref{eq: remaining terms marginal score 2} admits an alternative expression in terms of the conditional score $\nabla_{z_{N}}\log q_{\phi}(z_{N}\mid z_{N-1})$; see Appendix~\ref{section:appendix B} for the detailed derivations. With this representation, \eqref{eq: remaining terms marginal score 1} and~\eqref{eq: remaining terms marginal score 2} can be written as:
\begin{align}
    \label{eq: remaining terms conditional score 1}
    & A = \mathbb{E}_{N, z_{N-1}, z_{N}} \left[ g^{2}(t_{N}) s_{\theta}(z_{N}, t_{N}) \cdot \nabla_{z_{N}}\log q_{\phi}(z_{N} | z_{N - 1}) \right], \\
    \label{eq: remaining terms conditional score 2}
    & B = \mathbb{E}_{N, z_{N-1}, z_{N}} \left[ f(z_{N}, t_{N}) \cdot \nabla_{z_{N}}\log q_{\phi}(z_{N} | z_{N-1}) \right],
\end{align}
where
\begin{align}
    \label{eq: text closed form transition}
    \nabla_{z_{N}}\log q_{\phi}(z_{N} | z_{N-1}) = - \frac{U_{N - 1}}{\sigma_{N - 1}} = -\frac{U_{N - 1}}{g(t_{N - 1})\sqrt{\Delta t_{N - 1}}}.
\end{align}
Note that to have an accurate Gaussian approximation of the conditional score $\nabla_{z_{N}}\log q_{\phi}(z_{N} | z_{N-1})$, $\Delta t_{N - 1}$ must be sufficiently small.  However, since the conditional score scales as $1/\sqrt{\Delta t_{N-1}}$, this choice leads to terms of large magnitude and consequently high variance.

To mitigate this issue, we identify and remove the following two dominant \textit{mean-zero} terms from~\eqref{eq: remaining terms conditional score 1} and~\eqref{eq: remaining terms conditional score 2}  via first-order Taylor expansions of $s_{\theta}(z_{N}, t_{N})$ and $f(z_{N}, t_{N})$ about $(\mu_{N-1}, t_{N})$:
\begin{align}
    \label{eq: control variates}
    & s_{\theta}(\mu_{N-1}, t_{N})\frac{U_{N-1}}{\sigma_{N-1}},\; f(\mu_{N-1}, t_{N})\frac{U_{N-1}}{\sigma_{N-1}}.
\end{align}
These two terms are referred to as control variates in~\cite{NDSM}, as they have zero mean but large variance. In general, learnable coefficients can be introduced to scale these control variates, which can then be added back to the training loss to reduce its variance. In our setting, we simply remove these terms from the training loss. See Appendix~\ref{section:appendix B} for further details. Consequently, Eqs.~\eqref{eq: remaining terms conditional score 1} and~\eqref{eq: remaining terms conditional score 2} can be rewritten, respectively, as
\begin{align}
    \label{eq: remaining terms conditional score removed 1}
    & - \mathbb{E}_{N, z_{N-1}, z_{N}} \left[ g^{2}(t_{N}) \frac{U_{N-1}}{\sigma_{N-1}} \left[ s_{\theta}(z_{N}, t_{N}) - s_{\theta}(\mu_{N - 1}, t_{N}) \right] \right], \\
    \label{eq: remaining terms conditional score removed 2}
    & -\mathbb{E}_{N, z_{N-1}, z_{N}} \left[ \frac{U_{N-1}}{\sigma_{N-1}} \left[ f(z_{N}, t_{N}) - f(\mu_{N - 1}, t_{N}) \right] \right].
\end{align}


\subsection{Main Theorem}\label{subsection:theorem 1}

The preceding derivation leads to a computable, variance-reduced training objective for LNDSM-SGM, summarized in the following theorem. A detailed derivation is provided in Appendix~\ref{section:appendix B}.

\begin{theorem}
\label{thm:main}
    Let $\{z_{n}\}_{n=0}^{n_{f}}$ be the Markov process defined by the EM scheme~\eqref{eq:EM Markov process 1}~\eqref{eq:EM Markov process 2}. Let $N$ be a random variable uniformly distributed on $\{1, \cdots, n_{f}\}$ and independent of $z_{0}$ and the $U_{n}$'s. Then the cross-entropy term in VAE training loss can be approximately upper bounded by:
    \begin{align}
        \label{eq:main theorem text}
    \text{\normalfont CE}(q_{\phi}(z_{0}|x) \rVert p_{\theta}(z_{0}))
    \lesssim\; & H(q_{\phi}(z_{n_{f}})) + \frac{{T}}{2}\mathbb{E}_{N, z_{N}} \left[ g^{2}(t_{N}) \rVert s_{\theta}(z_{N}, t_{N}) \rVert_{2}^{2} \right] \nonumber \\
    & + {T}\mathbb{E}_{N, z_{N-1}, z_{N}} \left[ g^{2}(t_{N}) \frac{U_{N-1}}{\sigma_{N-1}} \left[ s_{\theta}(z_{N}, t_{N}) - s_{\theta}(\mu_{N - 1}, t_{N}) \right] \right] \nonumber \\
    & - {T}\mathbb{E}_{N, z_{N-1}, z_{N}} \left[ \frac{U_{N-1}}{\sigma_{N-1}} \left[ f(z_{N}, t_{N}) - f(\mu_{N - 1}, t_{N}) \right] \right].
    \end{align}
\end{theorem}
Substituting~\eqref{eq:main theorem text} into the variational autoencoder loss~\eqref{eq:new VAE-induced loss function} yields the final training objective. The encoder $q_{\phi}(z_{0} | x)$, the decoder $p_{\psi}(x | z_{0})$, and the score network $s_{\theta}(z_{t}, t)$ are trained jointly by minimizing this final loss. We refer to the resulting latent SGM with a nonlinear noising process as LNDSM-SGM.


\section{Experiments}\label{section:4}

Similar to \cite{NDSM}, we demonstrate the enhanced performance of our proposed LNDSM-SGM in learning high-dimensional structured distributions using the MNIST dataset \cite{MNIST} and the CIFAR-10 dataset \cite{cifar10}. The MNIST dataset is a collection of 60,000 handwritten digits from 0 to 9, stored as $28 \times 28$ grayscale images (Figure~\ref{fig:samples and frequency a}). The CIFAR-10 dataset is a collection of 60,000 images in 10 class, stored as $32 \times 32$ color images (Figure~\ref{fig:samples and frequency CIFAR-10 a}). Both datasets exhibit multimodal structures, in both the original space and, perhaps even more prominently, in a latent space, as each digit/class naturally forms at least one mode. For this reason, the reference distributions are chosen as Gaussian mixture models (GMM) with 10 Gaussian components, where each component corresponds to the mode associated with one digit/class.

For the MNIST dataset, we also consider a low-data regime, in which only 3,000 samples from the target distribution are used for training. We further conduct experiments on the Approx.-C2-MNIST dataset. This dataset is constructed by randomly rotating each digit by $180^\circ$ with probability $0.5$, followed by resizing the rotated digits to half of their original size. As a result, the large digits remain upright, while the small digits appear upside down (Figure~\ref{fig:samples and frequency approx-C2 a}). The dataset is referred to as Approx.-C2 because the samples can be viewed as drawn from a distribution that is approximately invariant under the discrete rotation group $C_{2}$. The reference distribution in this case is a GMM with two Gaussian components: one corresponding to the large, upright digits, and the other to the small, upside-down digits.

We evaluate the quality and diversity of generated samples using the Inception Score~\cite{IS} (IS; higher is better) and quantify the discrepancy between the generated and target distributions using the Fr\'echet Inception Distance~\cite{FID} (FID; lower is better). To compute the FID, we train a separate ResNet-18 classifier on each of the MNIST, Approx.-C2-MNIST, and CIFAR-10 datasets.


\subsection{Implementation Details}

Both the benchmark LSGM and our proposed LNDSM-SGM use the NVAE~\cite{NVAE} architecture as the backbone of the VAE and the NCSNpp~\cite{song2021scorebased} as the backbone of the score network. By maintaining the identical network architectures, we ensure that any observed performance differences are attributable solely to the proposed LNDSM training objective.
Given the limited computational resources, we simplify the NVAE architecture used in~\cite{LSGM} by reducing the number of channels in both the feature maps and the latent variables; see Appendix~\ref{section:appendix C} for details of the simplification and the NVAE building blocks. The dimensions of the latent variables determined by the NVAE encoder are $2 \times 8 \times 8 = 128$ for the MNIST dataset and $2 \times 16 \times 16 = 512$ for the CIFAR-10 dataset.

The NCSNpp-based score network is also simplified from the version used in \cite{LSGM}. For the MNIST (CIFAR-10) dataset, the number of base channels is reduced from 256 to 32 (512 to 64). The numbers of residual blocks in the U-Net downsampling and upsampling paths are reduced from 16 to 4 (24 to 4) and from 18 to 6 (27 to 6), respectively.
The numbers of trainable parameters for all networks are reported in Table~\ref{tab:numbers of trainable parameters}. We also include the score network used in~\cite{NDSM}, which is based on a U-Net~\cite{U-Net}. For the MNIST dataset, the simplifications to NVAE and NCSNpp are designed to ensure that their combined number of trainable parameters is comparable to that of the U-Net.
\begin{table}[ht]
    \centering
    \caption{Number of trainable parameters}
    \label{tab:numbers of trainable parameters}
    \begin{tabular}{ccc}
    \toprule
    Networks & \# parameters (MNIST) & \# parameters (CIFAR-10) \\
    \midrule
        NVAE                  &    207,648     &  299,518  \\
        NCSNpp                &    895,076     &  2,177,608  \\
        NVAE + NCSNpp         &    1,102,724   &  2,477,126  \\
        U-Net in NDSM-SGM     &    1,115,297   &  /  \\
    \bottomrule
    \end{tabular}
\end{table}

For the benchmark LSGM, we adopt the Variance-Preserving (VP) SDE~\cite{song2021scorebased} as the forward diffusion process in the latent space on all the datasets:
\begin{align}
    \label{eq:VP SDE in the latent space}
    \dd z_{t} = -\frac{1}{2}\beta(t)z_{t}\dd t + \sqrt{\beta(t)}\dd w_{t},
\end{align}
where $\beta(t)$ is a linear function on $[0, T]$ with $\beta(0) = 0.1$ and $\beta(T) = 20$.
The terminal time $T$ is chosen to be $1$. Following the experimental setup for MNIST in \cite{LSGM}, the VAE is pretrained for 200 epochs and used to initialize the VAE in the benchmark LSGM. The VAE and the score network are then trained jointly using the Adam optimizer~\cite{kingma2015adam} with a batch size of 120 on an NVIDIA V100 GPU with 32 GB of memory. The learning rates are $10^{-4}$ and $3 \times 10^{-4}$, respectively. The number of training epochs is 500 for the full-data regime (both MNIST and Approx.-C2-MNIST) and 1000 for the low-data regime. For the CIFAR-10 dataset, we pretrain the VAE for 400 epochs. We then jointly train the VAE and the score network using the same optimizer and learning rates, with a batch size of 32 on an NVIDIA RTX 6000 GPU with 24 GB of memory. The total training time is capped at 48 hours.

For our proposed LNDSM-SGM, the VAE encoder maps the training samples to the latent space, where a GMM reference distribution is fitted using the latent variables.
For the MNIST dataset, the VAE is initialized to be the VAE trained under the benchmark LSGM for 100 epochs, and its learning rate is set to $10^{-6}$, which is much smaller than $10^{-4}$. Since the entropy of the GMM reference distribution in the training loss has no closed-form expression, this initialization together with the small learning rate ensures that the latent variables remain sufficiently expressive and do not change significantly during training. Therefore, the fitted GMM in the latent space can be regarded as approximately fixed, and its entropy term is removed from the training loss. The terminal time of the latent space forward diffusion process~\eqref{eq:generic diffusion process in the latent space} is set to $T=1.5$, and the number of time steps in the EM scheme is set to 100. The VAE and the score network are trained jointly using the Adam optimizer with a batch size of 60 on an NVIDIA V100 GPU with 32 GB of memory for 50 epochs for both regimes.

For the CIFAR-10 dataset, we use a pretrained VAE and keep its parameters fixed throughout training. Consequently, we do not need to retain the gradients of the simulated latent variables with respect to the NVAE parameters, thereby reducing GPU memory usage. The terms involving the VAE parameters can therefore be removed from the training objective in~\eqref{eq:new VAE-induced loss function}, leaving only the second and third terms of the cross-entropy objective in~\eqref{eq:main theorem text}. To determine the terminal time of the forward diffusion process in the latent space, we first fit the target GMM and map the CIFAR-10 images to the latent space using the fixed VAE. We then simulate the latent variables up to different candidate terminal times and evaluate their average log-likelihood under the target GMM. Based on the results shown in Figure~\ref{fig:Average log-likelihood vs terminal times}, we set the terminal time to $T=0.5$. The EM scheme uses 50 time steps. Finally, we train the score network using the Adam optimizer with a batch size of 32 on an NVIDIA RTX 6000 GPU with 24 GB of memory. The total training time is capped at 48 hours.

\begin{figure}[htbp]
  \centering
  \includegraphics[width=0.75\linewidth]{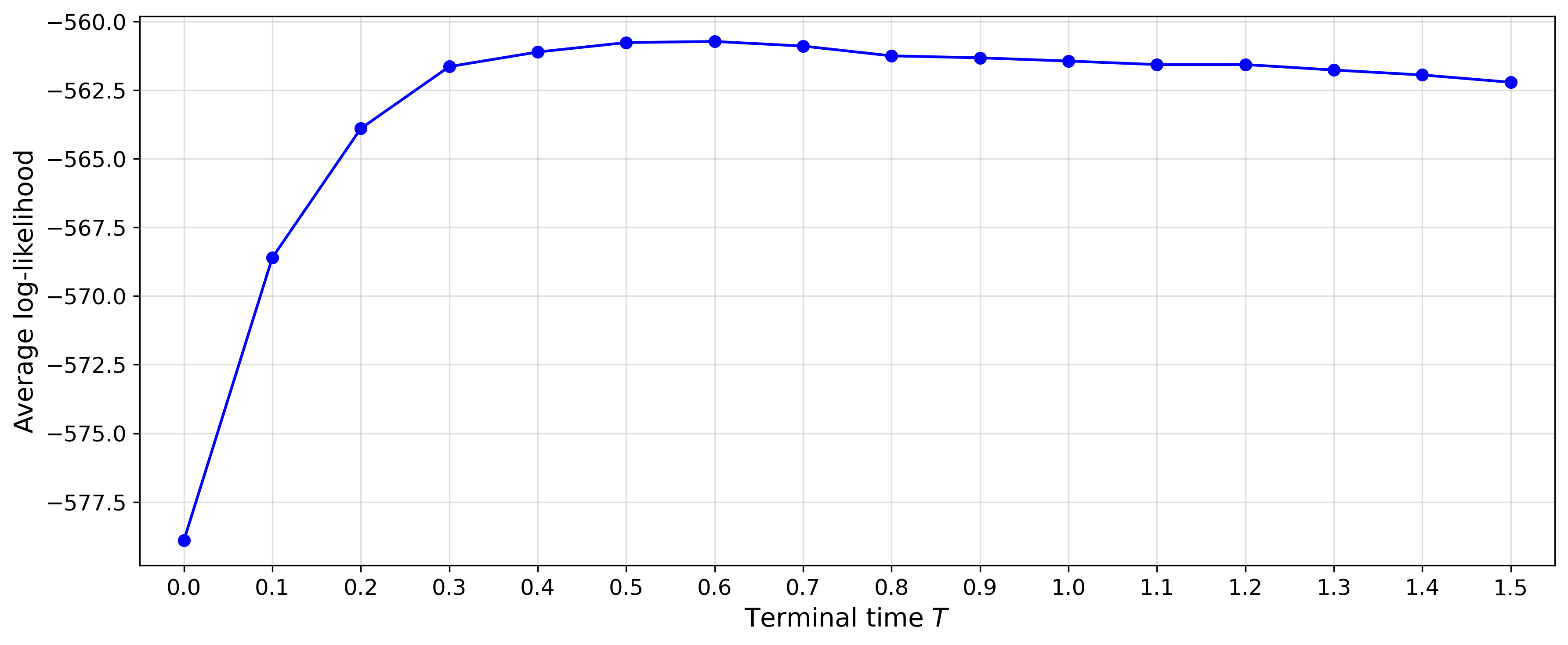}
  \caption{The average log-likelihood plateaus at a terminal time of $T=0.5$, indicating that the latent-space diffusion process has brought the latent variable distribution sufficiently close to the target GMM.}
  \label{fig:Average log-likelihood vs terminal times}
\end{figure}


\subsection{Results Comparison}

An important technical component of the benchmark LSGM is importance sampling (Appendix B in \cite{LSGM}). The expectation $\mathbb{E}_{t}$ in~\eqref{eq:cross entropy} has high variance if $t$ is uniformly sampled from $[0, T]$, which leads to ineffective training and poor sample quality, as reflected by large FID values. Importance sampling alleviates this issue by drawing $t$ from a suitably designed importance distribution. Although this approach is empirically effective, its derivation relies on the assumption that the latent target distribution is standard Gaussian. In contrast, for the proposed LNDSM-SGM, all trajectories generated by the EM scheme are directly used to compute the training loss~\eqref{eq:main theorem text}. As a result, no additional importance sampling is required for variance reduction.

A limitation of NDSM-SGM~\cite{NDSM} is its slow synthesis speed, which is dominated by the cost of simulating the forward SDE via the EM method. In~\eqref{eq:nonlinear forward SDE text}, the state variables are 784-dimensional vectors for the MNIST dataset, and the EM discretization uses 1000 steps. With a batch size of 60, the average simulation time is 0.35 seconds. Our proposed LNDSM-SGM achieves substantially faster synthesis. Variables in the SDE~\eqref{eq:generic diffusion process in the latent space} are 128-dimensional vectors in the latent space for the MNIST dataset. Because of the joint training setup, gradients of these variables with respect to the parameters of the VAE are stored. The latent representations allow for smaller number of EM discretization steps, which is 100. This choice is justified by the empirical fact that the probability flow ODE is usually much less complex in the latent space~\cite{LSGM}. The resulting average simulation time is 0.068 seconds with a batch size of 60.

Because of the EM scheme, each training epoch of  LNDSM-SGM is computationally more expensive compared to LSGM. For the MNIST (CIFAR-10) dataset, each epoch requires around 600 (1500) seconds, compared with 80 (300) seconds per epoch for the benchmark LSGM. However, all the latent variables generated by the EM scheme are used to evaluate the training loss~\eqref{eq:main theorem text} and are therefore not wasted. While LSGM only uses one latent variable to compute the training loss.

\begin{figure}[htbp]
  \centering
  
  \begin{subfigure}[t]{0.45\linewidth}
    \centering
    \includegraphics[width=\linewidth]{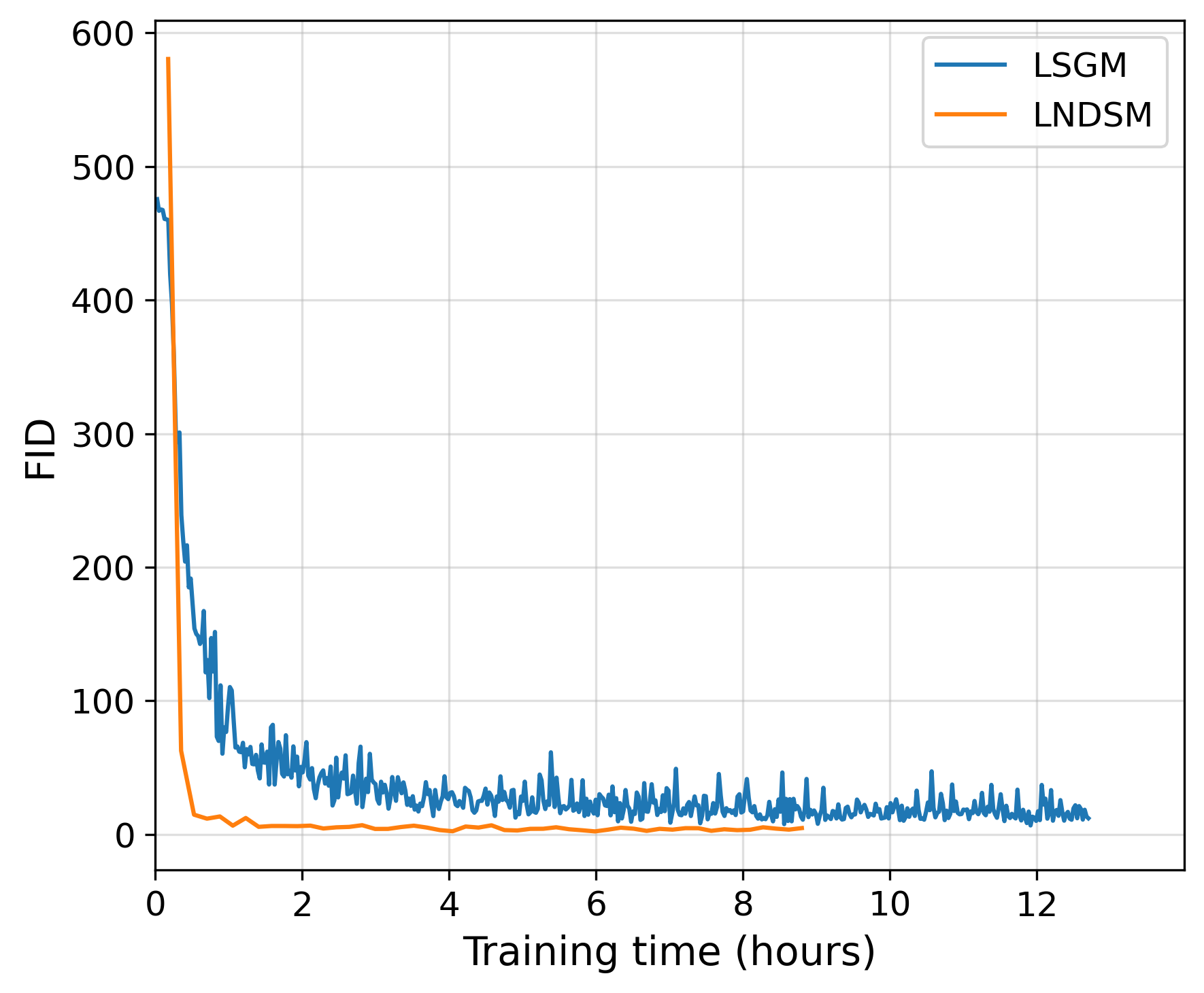}
    \caption{MNIST}
    \label{fig:training plot MNIST}
  \end{subfigure}
  \hspace{0.5em}
  \begin{subfigure}[t]{0.45\linewidth}
    \centering
    \includegraphics[width=\linewidth]{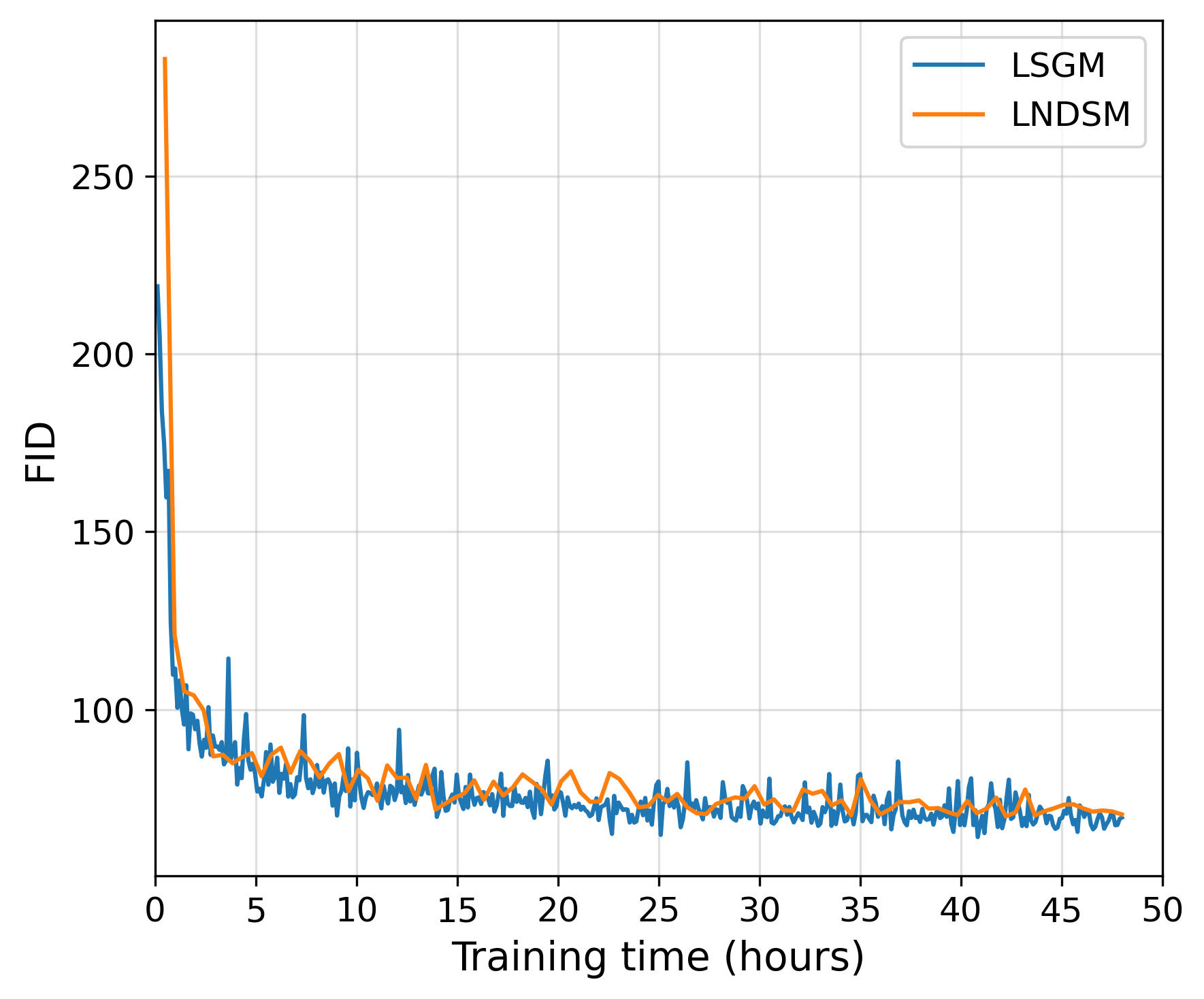}
    \caption{CIFAR-10}
    \label{fig:training plot CIFAR-10}
  \end{subfigure}
  
  \caption{(a): For the MNIST dataset, LNDSM-SGM only needs a few training epochs to convergence, whereas the benchmark LSGM requires 500 epochs. The total training time of the proposed LNDSM-SGM is lower. (b): For the CIFAR-10 dataset, the total training time is capped at 48 hours. LNDSM-SGM is trained for 100 training epochs to convergence. The benchmark LSGM is trained for 436 epochs to convergence.}
  \label{fig:training plot}
\end{figure}

After training, 10{,}000 samples are generated by simulating the reverse-time SDE or the corresponding probability flow ODE \cite{song2021scorebased}, with the marginal score replaced by the trained score network. For the MNIST dataset, these samples are used to compute the FID and IS. A snapshot (64 out of 10{,}000 samples) and the empirical fraction of different digits are visualized in Figures~\ref{fig:samples and frequency}, \ref{fig:samples and frequency 3000}, and~\ref{fig:samples and frequency approx-C2}. The samples generated by the benchmark LSGM exhibit mode imbalance (Figure~\ref{fig:samples and frequency e},~\ref{fig:samples and frequency 3000 e}, and~\ref{fig:samples and frequency approx-C2 e}). The empirical fractions of generated digits differ significantly from the target fractions in the training data (Figures~\ref{fig:samples and frequency d}, \ref{fig:samples and frequency 3000 d}, and~\ref{fig:samples and frequency approx-C2 d}). This imbalance is largely resolved by the proposed LNDSM-SGM. By introducing a GMM reference distribution to impose the multimodal structure, the fractions of generated digits closely match the target distribution, especially for the MNIST and the Approx.-C2-MNIST (Figure~\ref{fig:samples and frequency f} and~\ref{fig:samples and frequency approx-C2 f}). In the low-data MNIST setting, although some discrepancy remains (Figure~\ref{fig:samples and frequency 3000 f}), LNDSM-SGM captures the relative magnitude of digit frequencies more accurately than the benchmark LSGM. The KL divergence of the fraction of the training samples from the fraction of the generated samples is also included in the table.

\begin{figure}[htbp]
  \centering
  
  \begin{subfigure}[t]{0.24\linewidth}
    \centering
    \includegraphics[width=\linewidth]{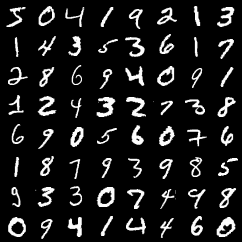}
    \caption{Training samples}
    \label{fig:samples and frequency a}
  \end{subfigure}
  \hspace{0.5em}
  \begin{subfigure}[t]{0.24\linewidth}
    \centering
    \includegraphics[width=\linewidth]{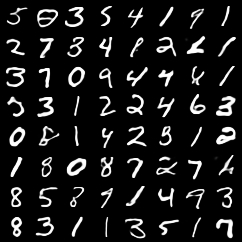}
    \caption{LSGM}
    \label{fig:samples and frequency b}
  \end{subfigure}
  \hspace{0.5em}
  \begin{subfigure}[t]{0.24\linewidth}
    \centering
    \includegraphics[width=\linewidth]{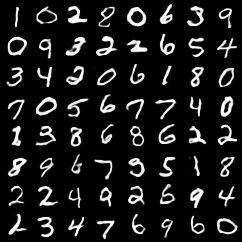}
    \caption{\textbf{LNDSM-SGM}}
    \label{fig:samples and frequency c}
  \end{subfigure}

  \vspace{0.5em}
  
  \begin{subfigure}[t]{0.24\linewidth}
    \centering
    \includegraphics[width=\linewidth]{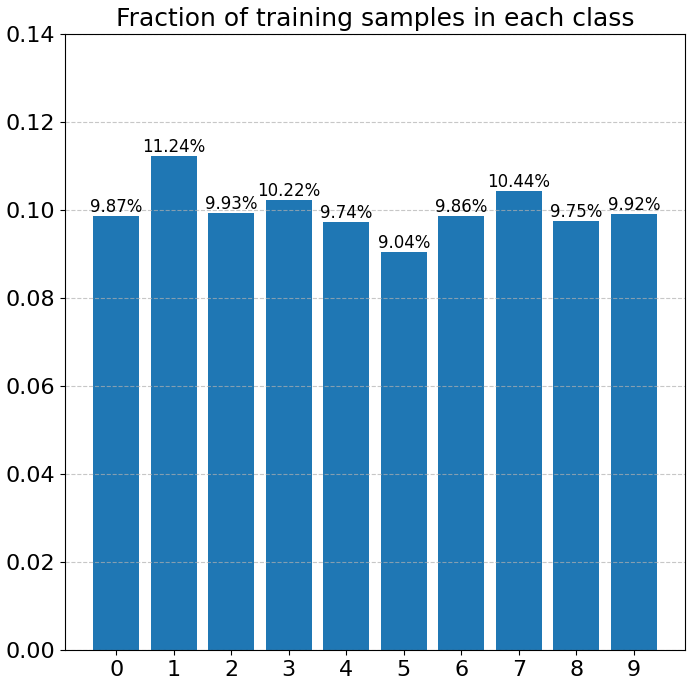}
    \caption{Training samples}
    \label{fig:samples and frequency d}
  \end{subfigure}
  \hspace{0.5em}
  \begin{subfigure}[t]{0.24\linewidth}
    \centering
    \includegraphics[width=\linewidth]{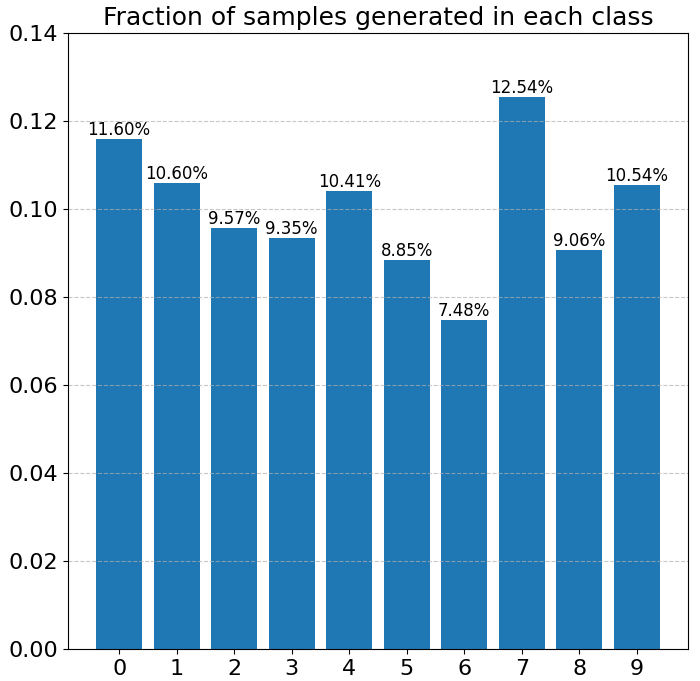}
    \caption{LSGM}
    \label{fig:samples and frequency e}
  \end{subfigure}
  \hspace{0.5em}
  \begin{subfigure}[t]{0.24\linewidth}
    \centering
    \includegraphics[width=\linewidth]{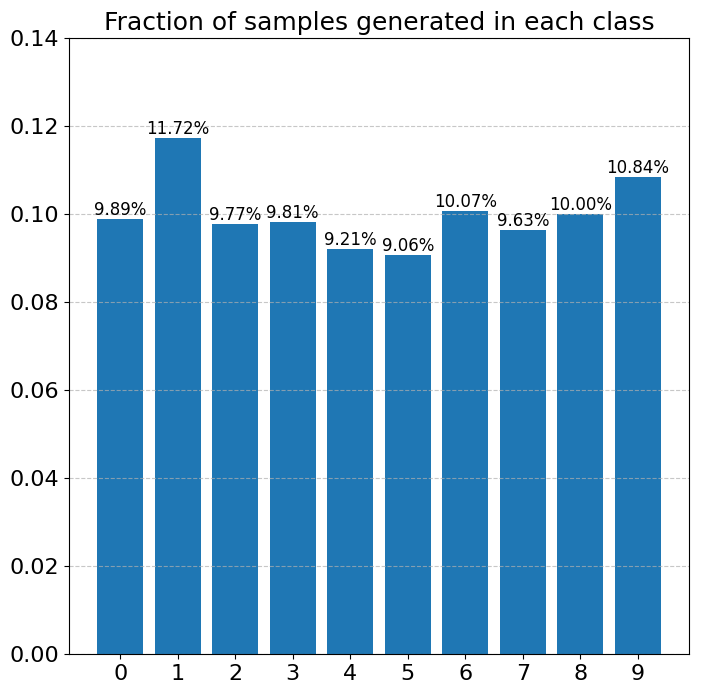}
    \caption{\textbf{LNDSM-SGM}}
    \label{fig:samples and frequency f}
  \end{subfigure}
  
  \caption{LSGM and LNDSM-SGM trained using the full MNIST dataset. (a) and (d): Snapshot and fraction of different digits of the training samples. (b) and (e): Snapshot and fraction of different digits of the 10,000 samples generated by the LSGM. The KL divergence from the fraction of the training samples to the fraction of the generated samples: 0.008. (c) and (f): Snapshot and fraction of different digits of the 10,000 samples generated by the LNDSM-SGM. The KL divergence of the fraction of the training samples from the fraction of the generated samples: 0.00114.}
  \label{fig:samples and frequency}
\end{figure}

The best FID values and the corresponding IS achieved during the training process are reported in Tables~\ref{tab:FID IS} and~\ref{tab:FID IS low data}. The exhibited FID and IS for NDSM-SGM on MNIST are taken from~\cite{NDSM}. The benchmark LSGM outperforms the NDSM-SGM in both FID and IS, which can be attributed to the mapping of high-dimensional data into a more compact and informative latent space that facilitates density estimation. The proposed LNDSM-SGM further improves upon LSGM, achieving substantially lower FID and higher IS on both MNIST and low-data MNIST. These results demonstrate the enhanced capability of LNDSM-SGM in learning structured target distributions.

\begin{table}[htbp]
    \centering
    \caption{FID and IS results on MNIST and CIFAR-10}
    \label{tab:fid-is}

    \begin{subtable}[t]{0.48\linewidth}
        \centering
        \caption{MNIST}
        \label{tab:FID IS}
        \begin{tabular}{cccc}
            \toprule
            Model & FID ($\downarrow$) & IS ($\uparrow$) & KL ($\downarrow$) \\
            \midrule
            NDSM-SGM~\cite{NDSM}  & 36.1         & 8.93 &           /  \\
            LSGM~\cite{LSGM}      & 6.7          & 9.42 &           0.008  \\
            LNDSM-SGM (ours) & \textbf{2.2} & \textbf{9.75} &  \textbf{0.00114}  \\
            \bottomrule
        \end{tabular}
    \end{subtable}
    \begin{subtable}[t]{0.48\linewidth}
        \centering
        \caption{CIFAR-10}
        \label{tab:FID IS CIFAR-10}
        \begin{tabular}{ccc}
            \toprule
            Model & FID ($\downarrow$) & KL ($\downarrow$) \\
            \midrule
            LSGM~\cite{LSGM}      & \textbf{64.2} &  0.294  \\
            LNDSM-SGM (ours) & 70.0          &  \textbf{0.227}  \\
            \bottomrule
        \end{tabular}
    \end{subtable}
\end{table}

For the CIFAR-10 dataset, the generated 10{,}000 samples are used to compute the FID. A snapshot (64 out of 10{,}000 samples) and the empirical fraction of different classes are visualized in Figure~\ref{fig:samples and frequency CIFAR-10}. The samples generated by the benchmark LSGM exhibit significant mode imbalance (Figure~\ref{fig:samples and frequency CIFAR-10 e}). $33.2\%$ generated samples are classified as the class "cat". The empirical fractions of generated classes differ significantly from the target fractions in the training data (Figure~\ref{fig:samples and frequency CIFAR-10 d}). This imbalance is slightly resolved by the proposed LNDSM-SGM. The fraction of generated cat images is reduced by $8.43\%$ and the fractions of most the other classes increase slightly, producing class fractions that more closely match the target distribution. The best FID values achieved during training are reported in Table~\ref{tab:FID IS CIFAR-10}. We observe that both models yield relatively high FID scores, due to the limited capacity of the simplified neural network architectures, which contain relatively few parameters. The proposed LNDSM-SGM is slightly outperformed by the benchmark LSGM. We attribute this performance gap to fixing the VAE parameters in LNDSM-SGM to reduce GPU memory usage.

\begin{figure}[htbp]
  \centering
  
  \begin{subfigure}[t]{0.24\linewidth}
    \centering
    \includegraphics[width=\linewidth]{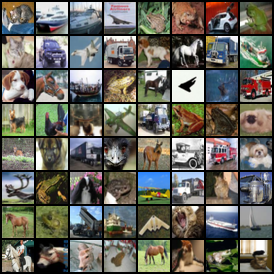}
    \caption{Training samples}
    \label{fig:samples and frequency CIFAR-10 a}
  \end{subfigure}
  \hspace{0.5em}
  \begin{subfigure}[t]{0.24\linewidth}
    \centering
    \includegraphics[width=\linewidth]{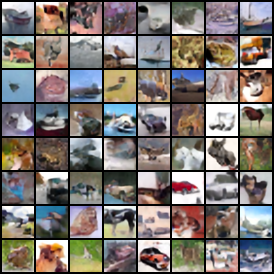}
    \caption{LSGM}
    \label{fig:samples and frequency CIFAR-10 b}
  \end{subfigure}
  \hspace{0.5em}
  \begin{subfigure}[t]{0.24\linewidth}
    \centering
    \includegraphics[width=\linewidth]{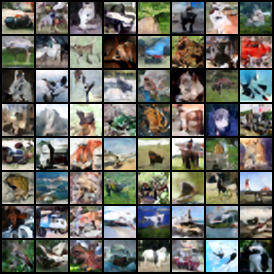}
    \caption{\textbf{LNDSM-SGM}}
    \label{fig:samples and frequency CIFAR-10 c}
  \end{subfigure}

  \vspace{0.5em}
  
  \begin{subfigure}[t]{0.24\linewidth}
    \centering
    \includegraphics[width=\linewidth]{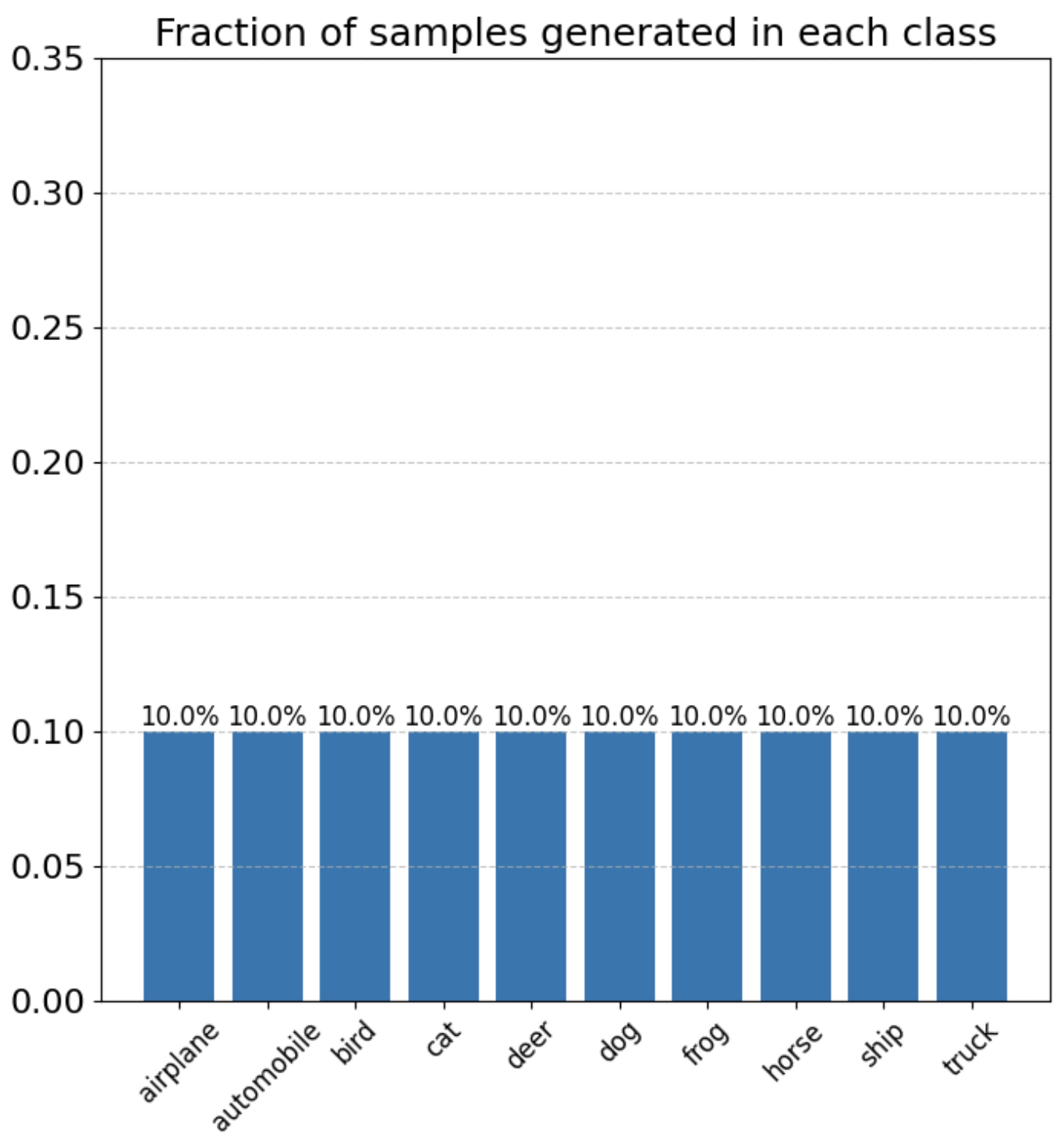}
    \caption{Training samples}
    \label{fig:samples and frequency CIFAR-10 d}
  \end{subfigure}
  \hspace{0.5em}
  \begin{subfigure}[t]{0.24\linewidth}
    \centering
    \includegraphics[width=\linewidth]{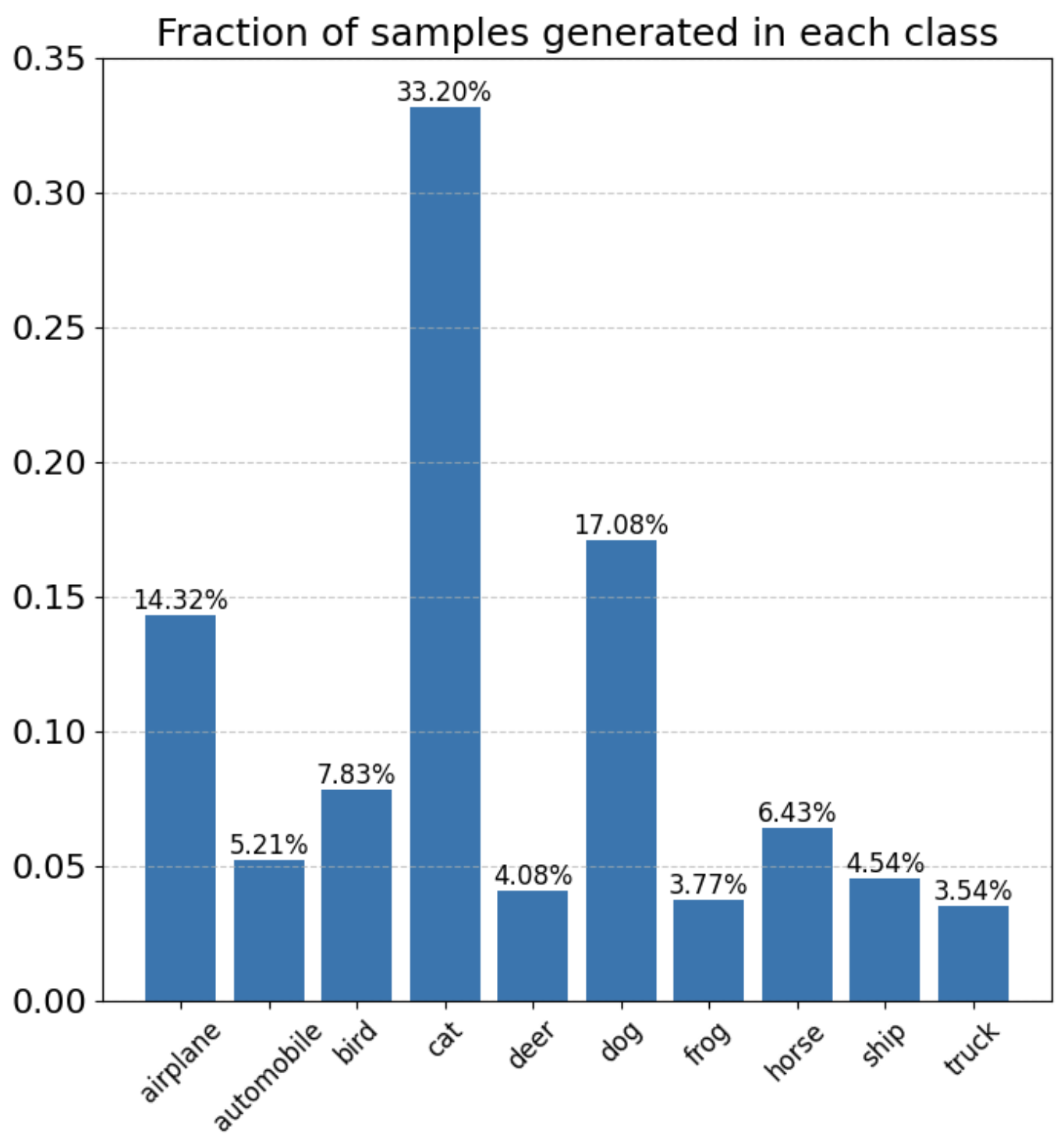}
    \caption{LSGM}
    \label{fig:samples and frequency CIFAR-10 e}
  \end{subfigure}
  \hspace{0.5em}
  \begin{subfigure}[t]{0.24\linewidth}
    \centering
    \includegraphics[width=\linewidth]{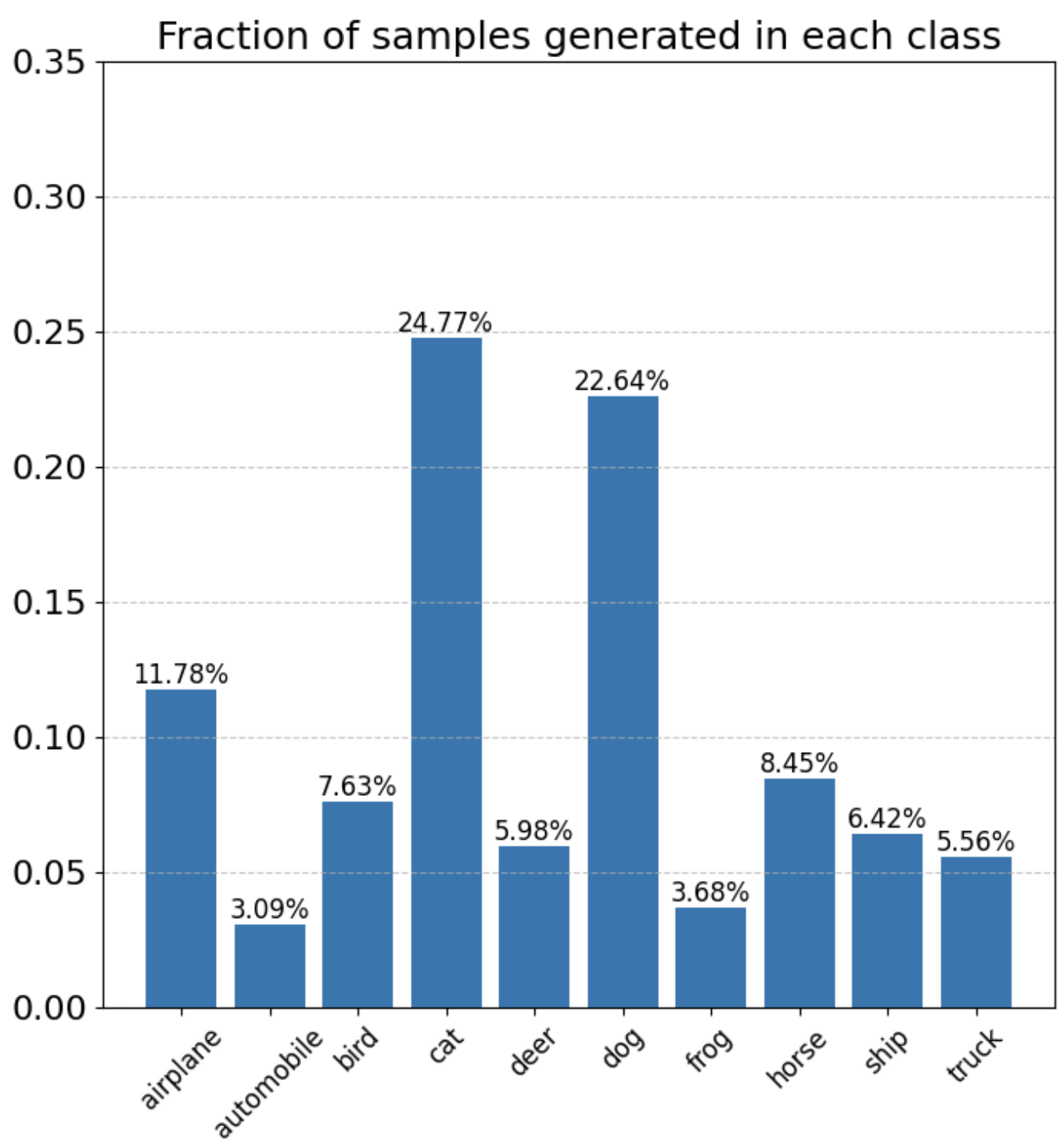}
    \caption{\textbf{LNDSM-SGM}}
    \label{fig:samples and frequency CIFAR-10 f}
  \end{subfigure}
  
  \caption{LSGM and LNDSM-SGM trained using the CIFAR-10 dataset. (a) and (d): Snapshot and fraction of different classes of the training samples. (b) and (e): Snapshot and fraction of different classes of the 10,000 samples generated by the LSGM. The KL divergence from the fraction of the training samples to the fraction of the generated samples: 0.294. (c) and (f): Snapshot and fraction of different classes of the 10,000 samples generated by the LNDSM-SGM. The KL divergence from the fraction of the training samples to the fraction of the generated samples: 0.227.}
  \label{fig:samples and frequency CIFAR-10}
\end{figure}


\section{Conclusion}\label{section:conclusion}

We presented latent nonlinear denoising score matching (LNDSM), which integrates nonlinear diffusions and structured latent priors into the latent score-based generative modeling framework. By reformulating the VAE-induced cross-entropy via an Euler–Maruyama discretization and conditional Gaussian transitions, and by removing two variance-exploding control variates, LNDSM yields a numerically stable training objective that does not require importance sampling in time. Using Gaussian mixture models as latent reference distributions, LNDSM-SGM effectively encodes multimodal and approximately symmetric structure in the latent space. Experiments on MNIST and Approx.-C2-MNIST, including low-data regimes, together with experiments on CIFAR-10, show that LNDSM-SGM achieves substantially lower Fr\'{e}chet inception distance, higher inception score, improved mode balance, and reduced overall training time compared to LSGM and NDSM-SGM, indicating that nonlinear latent diffusions with structured latent priors provide an efficient and powerful approach for learning structured data distributions.


\section*{Acknowledgments}
This material is based upon work supported by the U.S. National Science Foundation under Awards No.~DMS-2502900 and DMS-2540370, and by the Air Force Office of Scientific Research (AFOSR) under Grant No.~FA9550-25-1-0079.
\clearpage


\bibliographystyle{plain}
\bibliography{references}
\clearpage


\appendix


\section{Lemma about Cross Entropy}\label{section:appendix A}

The following lemma is an extension of the theorem in Appendix~A of~\cite{LSGM} to a generic diffusion process. Specifically, the time interval is extended from $[0, 1]$ to $[0, T]$. Accordingly, the proof follows the same structure as in~\cite{LSGM}, prior to the specialization to a linear drift coefficient. For completeness, we provide full derivations and expand the steps that were abbreviated in the original proof.

\begin{lemma}
\label{lemma:lemma1}
    Let $q(z_{0})$ and $p(z_{0})$ be the densities of two distributions defined in the latent space. Consider the generic diffusion process
    \begin{align}
        \dd z_{t} = f(z_{t}, t)\dd t + g(t)\dd w_{t},
    \end{align}
    where the drift $f(z_{t}, t)$ is not necessarily linear in $z_{t}$. Denote the corresponding marginal densities at time $t\in[0,T]$ by $q(z_{t})$ and $p(z_{t})$. Assume that $q(z_{T}) = p(z_{T})$, and that $\log q(z_{t})$ and $\log p(z_{t})$ are smooth functions with at most polynomial growth as $z_{t}\to \pm\infty$. Then the cross entropy $\text{\normalfont CE}(q(z_0) \rVert p(z_0))$ between $q(z_{0})$ and $p(z_{0})$ satisfies
    \begin{align*}
        \text{\normalfont CE}(q(z_0) \rVert p(z_0)) =   H(q(z_{T})) + \int_{0}^{T}\mathbb{E}_{q(z_{t})} & \left[ \frac{1}{2}g^{2}(t)\rVert \nabla_{z_{t}}\log p(z_{t}) \rVert_{2}^{2} + f^{T}(z_{t}, t)\nabla_{z_{t}}\log q(z_{t})\right.\\
        &~~\left.- g^{2}(t)[\nabla_{z_{t}}\log q(z_{t})]^{T}\nabla_{z_{t}}\log p(z_{t})\right] \dd t.
    \end{align*}
\end{lemma}
\begin{proof}
The densities $q(z_{t})$ and $p(z_{t})$ satisfy the Fokker–Planck equations
\begin{align}
    \label{eq:Fokker-Planck-1}
    & \frac{\partial q(z_{t})}{\partial t} = \nabla_{z_{t}} \left[ \frac{1}{2}g^{2}(t)q(z_{t})\nabla_{z_{t}}\log q(z_{t}) - f(z_{t}, t)q(z_{t}) \right] := \nabla_{z_{t}}[h_{q}(z_{t}, t)q(z_{t})], \\
    \label{eq:Fokker-Planck-2}
    & \frac{\partial p(z_{t})}{\partial t} = \nabla_{z_{t}} \left[ \frac{1}{2}g^{2}(t)p(z_{t})\nabla_{z_{t}}\log p(z_{t}) - f(z_{t}, t)p(z_{t}) \right] := \nabla_{z_{t}}[h_{p}(z_{t}, t)p(z_{t})].
\end{align}

Rewrite the cross entropy $\text{CE}(q(z_{0}) \rVert p(z_{0}))$ by adding and subtracting an extra term:

\begin{align}
    \text{CE}(q(z_{0}) \rVert p(z_{0})) &= \text{CE}(q(z_{T}) \rVert p(z_{T})) - \text{CE}(q(z_{T}) \rVert p(z_{T})) + \text{CE}(q(z_{0}) \rVert p(z_{0})) \nonumber \\
    &= \text{CE}(q(z_{T}) \rVert p(z_{T})) - \int_{0}^{T}\frac{\partial}{\partial t}\text{CE}(q(z_{t}) \rVert p(z_{t})) \dd t \nonumber \\
    \label{eq:cross entropy integration}
    &= H(q(z_{T})) - \int_{0}^{T}\frac{\partial}{\partial t}\text{CE}(q(z_{t}) \rVert p(z_{t})) \dd t.
\end{align}

The third equality uses the assumption $q(z_{T}) = p(z_{T})$, and $H(q(z_{T}))$ is the entropy of $q(z_{T})$.

Next, for the time derivative $\frac{\partial}{\partial t}\text{CE}(q(z_{t}) \rVert p(z_{t}))$ in~\eqref{eq:cross entropy integration}:
\begin{align}
    \frac{\partial}{\partial t}\text{CE}(q(z_{t}) \rVert p(z_{t})) &= \frac{\partial}{\partial t}\int -q(z_{t})\log p(z_{t})\dd z_{t} \nonumber \\
    &= -\int \left[ \frac{\partial q(z_{t})}{\partial t}\log p(z_{t}) + \frac{q(z_{t})}{p(z_{t})}\frac{\partial p(z_{t})}{\partial t} \right]\dd z_{t} \nonumber \\
    \label{eq:cross entropy derivative}
    &= -\int \left[ \nabla_{z_{t}} \left[ h_{q}(z_{t}, t)q(z_{t}) \right]\log p(z_{t}) + \frac{q(z_{t})}{p(z_{t})}\nabla_{z_{t}}\left[ h_{p}(z_{t}, t)p(z_{t}) \right] \right]\dd z_{t}.
\end{align}
The third equality is due to the Fokker-Planck equations~\eqref{eq:Fokker-Planck-1} and~\eqref{eq:Fokker-Planck-2}.

Since $q(z_{t})$ and $p(z_{t})$ are probability densities, they are normalized and therefore satisfy $q(z_{t}),p(z_{t}) \to 0$ as $z_{t}\to \pm\infty$. By assumption, $\log q(z_{t})$ and $\log p(z_{t})$ are smooth functions with at most polynomial growth as $z_{t}\to \pm\infty$. Consequently, their gradients $\nabla_{z_{t}}\log q(z_{t})$ and $\nabla_{z_{t}}\log p(z_{t})$ are also smooth and exhibit at most polynomial growth. Moreover, $q(z_{t})$ and $p(z_{t})$ decay at least exponentially fast. Therefore, the following limits hold:
\begin{align}
    & \log p(z_{t})h_{q}(z_{t}, t)q(z_{t}) = \log p(z_{t})[\frac{1}{2}g^{2}(t)\nabla_{z_{t}}\log q(z_{t}) - f(z_{t}, t)]q(z_{t}) \rightarrow 0, z_{t}\rightarrow \pm\infty, \\
    & \log q(z_{t})h_{p}(z_{t}, t)p(z_{t}) = \log q(z_{t})[\frac{1}{2}g^{2}(t)\nabla_{z_{t}}\log p(z_{t}) - f(z_{t}, t)]p(z_{t}) \rightarrow 0, z_{t}\rightarrow \pm\infty.
\end{align}
Applying integration by parts yields
\begin{align}
    & \int \nabla_{z_{t}}[h_{q}(z_{t}, t)q(z_{t})]\log p(z_{t})\dd z_{t} = \int \log p(z_{t}) \dd [h_{q}(z_{t}, t)q(z_{t})] \nonumber \\
    =& \log p(z_{t})h_{q}(z_{t}, t)q(z_{t})|_{z_{t}\rightarrow \pm \infty} - \int [h_{q}(z_{t}, t)]^{T}q(z_{t})\nabla_{z_{t}}\log p(z_{t})\dd z_{t} \nonumber \\
    \label{eq:integral by partition 1}
    =& - \int [h_{q}(z_{t}, t)]^{T}q(z_{t})\nabla_{z_{t}}\log p(z_{t})\dd z_{t}.
\end{align}
Similarly,
\begin{align}
    & \int \frac{q(z_{t})}{p(z_{t})}\nabla_{z_{t}}[h_{p}(z_{t}, t)p(z_{t})] \dd z_{t} = \int \frac{q(z_{t})}{p(z_{t})}\dd[h_{p}(z_{t}, t)p(z_{t})] \nonumber \\
    =& \frac{q(z_{t})}{p(z_{t})}h_{p}(z_{t}, t)p(z_{t})|_{z_{t}\rightarrow \pm \infty} - \int [h_{p}(z_{t}, t)]^{T}p(z_{t})\nabla_{z_{t}}\frac{q(z_{t})}{p(z_{t})}\dd z_{t} \nonumber \\
    \label{eq:integral by partition 2}
    =& - \int [h_{p}(z_{t}, t)]^{T}p(z_{t})\nabla_{z_{t}}\frac{q(z_{t})}{p(z_{t})}\dd z_{t}.
\end{align}

Substituting these into~\eqref{eq:cross entropy derivative} gives
\begin{align}
    \label{eq:cross entropy derivative 2}
    & \frac{\partial}{\partial t}\text{CE}(q(z_{t}) \rVert p(z_{t})) = \int [h_{q}(z_{t}, t)]^{T}q(z_{t})\nabla_{z_{t}}\log p(z_{t}) \dd z_{t} + \int [h_{p}(z_{t}, t)]^{T}p(z_{t})\nabla_{z_{t}}\frac{q(z_{t})}{p(z_{t})} \dd z_{t}.
\end{align}
Using the log-derivative identity, we obtain
\begin{align}
    & \int [h_{p}(z_{t}, t)]^{T}p(z_{t})\nabla_{z_{t}}\frac{q(z_{t})}{p(z_{t})} \dd z_{t} \nonumber \\
    =& \int [h_{p}(z_{t}, t)]^{T}q(z_{t})\frac{p(z_{t})}{q(z_{t})}\nabla_{z_{t}}\frac{q(z_{t})}{p(z_{t})} \dd z_{t} = \int [h_{p}(z_{t}, t)]^{T}q(z_{t})\nabla_{z_{t}}\log\frac{q(z_{t})}{p(z_{t})}\dd z_{t} \nonumber \\
    \label{eq:log derivative trick}
    =& \int [h_{p}(z_{t}, t)]^{T}q(z_{t})\nabla_{z_{t}}\log q(z_{t})\dd z_{t} - \int [h_{p}(z_{t}, t)]^{T}q(z_{t})\nabla_{z_{t}}\log p(z_{t})\dd z_{t}.
\end{align}
Apply~\eqref{eq:log derivative trick} and expressions of $h_{p}(z_{t}, t)$ and $h_{q}(z_{t}, t)$ to~\eqref{eq:cross entropy derivative 2}:
\begin{align}
    & \frac{\partial}{\partial t}\text{CE}(q(z_{t}) \rVert p(z_{t})) \nonumber \\
    =& \int q(z_{t}) \{ [h_{q}(z_{t}, t)]^{T}\nabla_{z_{t}}\log p(z_{t}) + [h_{p}(z_{t}, t)]^{T}\nabla_{z_{t}}\log q(z_{t}) - [h_{p}(z_{t}, t)]^{T}\nabla_{z_{t}}\log p(z_{t}) \} \dd z_{t} \nonumber \\
    =& \int q(z_{t})\{\frac{1}{2}g^{2}(t)[\nabla_{z_{t}}\log q(z_{t})]^{T}\nabla_{z_{t}}\log p(z_{t}) - f^{T}(z_{t}, t)\nabla_{z_{t}}\log p(z_{t}) \nonumber \\
    &\;\;\;\;\;\;\;\;\;\;\;+ \frac{1}{2}g^{2}(t)[\nabla_{z_{t}}\log p(z_{t})]^{T}\nabla_{z_{t}}\log q(z_{t}) - f^{T}(z_{t}, t)\nabla_{z_{t}}\log q(z_{t}) \nonumber \\
    &\;\;\;\;\;\;\;\;\;\;\;- \frac{1}{2}g^{2}(t)[\nabla_{z_{t}}\log p(z_{t})]^{T}\nabla_{z_{t}}\log p(z_{t}) + f^{T}(z_{t}, t)\nabla_{z_{t}}\log p(z_{t})\}\dd z_{t} \nonumber \\
    \label{eq:cross entropy derivative 3}
    =& \int q(z_{t}) \left[ - \frac{1}{2}g^{2}(t)\rVert \nabla_{z_{t}}\log p(z_{t}) \rVert_{2}^{2} - f^{T}(z_{t}, t)\nabla_{z_{t}}\log q(z_{t}) + g^{2}(t)[\nabla_{z_{t}}\log q(z_{t})]^{T}\nabla_{z_{t}}\log p(z_{t}) \right] \dd z_{t}.
\end{align}
Finally, substituting~\eqref{eq:cross entropy derivative 3} into~\eqref{eq:cross entropy integration} proves the result:
\begin{align}
    &\text{CE}(q(z_{0}) \rVert p(z_{0})) \nonumber \\
    =& H(q(z_{T})) - \int_{0}^{T}\frac{\partial}{\partial t}\text{CE}(q(z_{t}) \rVert p(z_{t}))\dd t \nonumber \\
    =& H(q(z_{T})) + \int_{0}^{T}\mathbb{E}_{q(z_{t})} \left[ \frac{1}{2}g^{2}(t)\rVert \nabla_{z_{t}}\log p(z_{t}) \rVert_{2}^{2} + f^{T}(z_{t}, t)\nabla_{z_{t}}\log q(z_{t}) - g^{2}(t)[\nabla_{z_{t}}\log q(z_{t})]^{T}\nabla_{z_{t}}\log p(z_{t}) \right] \dd t.
\end{align}

\end{proof}


\section{Proof of Theorem~\ref{thm:main}}\label{section:appendix B}

\begin{reptheorem}{thm:main}
    Let $\{z_{n}\}_{n=0}^{n_{f}}$ be the Markov process defined by the EM scheme~\eqref{eq:EM Markov process 1}~\eqref{eq:EM Markov process 2}. Let $N$ be a random variable uniformly distributed on $\{1, \cdots, n_{f}\}$ and independent of $z_{0}$ and the $U_{n}$'s. Then the cross-entropy term in VAE training loss can be approximately upper bounded by:
    \begin{align}
        \label{eq:main theorem appendix}
    \text{\normalfont CE}(q_{\phi}(z_{0}|x) \rVert p_{\theta}(z_{0}))
    \lesssim\; & H(q_{\phi}(z_{n_{f}})) + \frac{{T}}{2}\mathbb{E}_{N, z_{N}} \left[ g^{2}(t_{N}) \rVert s_{\theta}(z_{N}, t_{N}) \rVert_{2}^{2} \right] \nonumber \\
    & + {T}\mathbb{E}_{N, z_{N-1}, z_{N}} \left[ g^{2}(t_{N}) \frac{U_{N-1}}{\sigma_{N-1}} \left[ s_{\theta}(z_{N}, t_{N}) - s_{\theta}(\mu_{N - 1}, t_{N}) \right] \right] \nonumber \\
    & - {T}\mathbb{E}_{N, z_{N-1}, z_{N}} \left[ \frac{U_{N-1}}{\sigma_{N-1}} \left[ f(z_{N}, t_{N}) - f(\mu_{N - 1}, t_{N}) \right] \right].
    \end{align}
\end{reptheorem}

\begin{proof}

View the densities $q_{\phi}(z_{0}\mid x)$ and $p_{\theta}(z_{0})$ as $q(z_{0})$ and $p(z_{0})$ in Lemma~\ref{lemma:lemma1}. By adding and subtracting the term $\frac{1}{2}g^{2}(t)\rVert \nabla_{z_{t}}\log q_{\phi}(z_{t}\mid x) \rVert_{2}^{2}$, the cross-entropy term in the loss~\eqref{eq:new VAE-induced loss function} can be written as

\noindent
\begin{align}
    \label{eq: cross entropy 3.1}
    & \text{CE}(q_{\phi}(z_{0}|x) \rVert p_{\theta}(z_{0})) \nonumber \\
    =\; & H(q_{\phi}(z_{T}|x)) \nonumber \\
    & + \int_{0}^{T}\mathbb{E}_{q_{\phi}(z_{t}|x)} \left[\frac{1}{2}g^{2}(t)\rVert \nabla_{z_{t}}\log p_{\theta}(z_{t}) \rVert_{2}^{2} + f^{T}(z_{t}, t)\nabla_{z_{t}}\log q_{\phi}(z_{t}|x) - g^{2}(t)\left[\nabla_{z_{t}}\log q_{\phi}(z_{t}|x)\right]^{T}\nabla_{z_{t}}\log p_{\theta}(z_{t}) \right] \dd t \nonumber \\
    =\; & H(q_{\phi}(z_{T}|x)) \nonumber \\
    & + \frac{1}{2}\int_{0}^{T}g^{2}(t)\mathbb{E}_{q_{\phi}(z_{t}|x)}\left[ \rVert \nabla_{z_{t}}\log q_{\phi}(z_{t} | x)  - \nabla_{z_{t}}\log p_{\theta}(z_{t}) \rVert_{2}^{2} \right] \dd t \nonumber \\
    & + \frac{1}{2}\int_{0}^{T}\mathbb{E}_{q_{\phi}(z_{t}|x)}\left[ \left[2f(z_{t}, t) - g^{2}(t)\nabla_{z_{t}}\log q_{\phi}(z_{t}|x) \right]^{T}\nabla_{z_{t}}\log q_{\phi}(z_{t} | x) \right]\dd t \nonumber \\
    \le\; & H(q_{\phi}(z_{T}|x)) \nonumber \\
    & + \frac{1}{2}\int_{0}^{T}g^{2}(t)\mathbb{E}_{q_{\phi}(z_{t}|x)}\left[ \rVert \nabla_{z_{t}}\log q_{\phi}(z_{t} | x)  - s_{\theta}(z_{t}, t) \rVert_{2}^{2} \right]\dd t \nonumber \\
    & + \frac{1}{2}\int_{0}^{T}\mathbb{E}_{q_{\phi}(z_{t}|x)}\left[ \left[2f(z_{t}, t) - g^{2}(t)\nabla_{z_{t}}\log q_{\phi}(z_{t}|x) \right]^{T}\nabla_{z_{t}}\log q_{\phi}(z_{t} | x) \right]\dd t.
\end{align}
The second equality follows from completing the square after adding and subtracting the extra term. Here $\nabla_{z_{t}}\log p_{\theta}(z_{t})$ is the score of the marginal distribution at time $t$, which is parameterized by the score network $s_{\theta}(z_{t}, t)$. Combined with the Markov process~\eqref{eq:EM Markov process 1}~\eqref{eq:EM Markov process 2}, the cross entropy~\eqref{eq: cross entropy 3.1} can be approximatedly bounded by:

\noindent
\begin{align}
    \label{eq: cross entropy 3.2}
    & \text{CE}(q_{\phi}(z_{0}|x) \rVert p_{\theta}(z_{0})) \nonumber \\
    \lesssim\; & H(q_{\phi}(z_{n_{f}})) \nonumber \\
    & + \frac{{T}}{2}\mathbb{E}_{N, z_{N}}\left[ g^{2}(t_{N}) \rVert \nabla_{z_{N}}\log q_{\phi}(z_{N}) - s_{\theta}(z_{N}, t_{N}) \rVert_{2}^{2} \right] \nonumber \\
    & + \frac{{T}}{2}\mathbb{E}_{N, z_{N}}\left[ \left[ 2f(z_{N}, t_{N}) - g^{2}(t_{N})\nabla_{z_{N}}\log q_{\phi}(z_{N}) \right]^{T} \nabla_{z_{N}}\log q_{\phi}(z_{N}) \right].
\end{align}
Here, the continuous time variable $t$ is replaced by the discrete time steps $t_{N}$, and the latent variable $z_{t}$ is replaced by $z_{N}$ in the discrete Markov process. Split the second squared norm term in~\eqref{eq: cross entropy 3.2}, also consider $N$ takes some value $n$ for simplicity:
\begin{align}
    \label{eq: splitted squared norm term}
    & \frac{{T}}{2}\mathbb{E}_{z_{n}}\left[ g^{2}(t_{n}) \rVert \nabla_{z_{n}}\log q_{\phi}(z_{n}) - s_{\theta}(z_{n}, t_{n}) \rVert_{2}^{2} \right] \nonumber \\
    =\; & \frac{{T}}{2}\mathbb{E}_{z_{n}}\left[ g^{2}(t_{n}) \rVert s_{\theta}(z_{n}, t_{n}) \rVert_{2}^{2} \right] - {T}\mathbb{E}_{z_{n}}\left[ g^{2}(t_{n})s_{\theta}(z_{n}, t_{n})\cdot \nabla_{z_{n}}\log q_{\phi}(z_{n}) \right] + \frac{{T}}{2}\mathbb{E}_{z_{n}}\left[ g^{2}(t_{n}) \rVert \nabla_{z_{n}}\log q_{\phi}(z_{n}) \rVert_{2}^{2} \right].
\end{align}
The second inner product term in~\eqref{eq: splitted squared norm term} can be written in integral as
\begin{align}
    \label{eq: splitted squared norm term second term integral}
    & \mathbb{E}_{z_{n}}\left[ g^{2}(t_{n})s_{\theta}(z_{n}, t_{n})\cdot \nabla_{z_{n}}\log q_{\phi}(z_{n}) \right] \nonumber \\
    =\; & g^{2}(t_{n}) \int s_{\theta}(z_{n}, t_{n}) \cdot q_{\phi}(z_{n}) \nabla_{z_{n}} \log q_{\phi}(z_{n}) \dd z_{n} \nonumber \\
    =\; & g^{2}(t_{n}) \int s_{\theta}(z_{n}, t_{n}) \cdot \nabla_{z_{n}} q_{\phi}(z_{n}) \dd z_{n} \nonumber \\
    =\; & g^{2}(t_{n}) \int s_{\theta}(z_{n}, t_{n}) \cdot \nabla_{z_{n}} \left[ \int \cdots \int \int q_{\phi}(z_{n} | z_{n - 1}) \cdots q_{\phi}(z_{1} | z_{0}) q_{\phi}(z_{0}) \dd z_{0} \dd z_{1} \cdots \dd z_{n - 1} \right] \dd z_{n} \nonumber \\
    =\; & g^{2}(t_{n}) \int \cdots \int \int \int \left[ s_{\theta}(z_{n}, t_{n}) \cdot \nabla_{z_{n}} \log q_{\phi}(z_{n} | z_{n - 1}) \right] q_{\phi}(z_{n} | z_{n - 1}) \cdots q_{\phi}(z_{1} | z_{0}) q_{\phi}(z_{0}) \dd z_{0} \dd z_{1} \cdots \dd z_{n - 1} \dd z_{n}
\end{align}

The second and fourth equality use the identity $x\;d\log x = x\frac{d\log x}{dx}dx = dx$. Since $z_{0},\ldots,z_{n}$ form a Markov chain, they satisfy
\begin{align}
    \label{eq: Markov property}
    & q_{\phi}(z_{n} | z_{n - 1}) \cdots q_{\phi}(z_{1} | z_{0}) q_{\phi}(z_{0}) = q_{\phi}(z_{n}, z_{n - 1}, \cdots, z_{1}, z_{0}).
\end{align}
Substituting~\eqref{eq: Markov property} into~\eqref{eq: splitted squared norm term second term integral} yields
\begin{align}
    & \mathbb{E}_{z_{n}}\left[ g^{2}(t_{n})s_{\theta}(z_{n}, t_{n})\cdot \nabla_{z_{n}}\log q_{\phi}(z_{n}) \right] \nonumber \\
    =\; & g^{2}(t_{n}) \int \int s_{\theta}(z_{n}, t_{n}) \cdot \nabla_{z_{n}} \log q_{\phi}(z_{n} | z_{n - 1}) \left[ \int \int q_{\phi}(z_{n}, z_{n - 1}, \cdots, z_{1}, z_{0}) \dd z_{0} \dd z_{1} \cdots \dd z_{n-2} \right] \dd z_{n - 1} \dd z_{n}.
\end{align}
The inner product $s_{\theta}(z_{n}, n) \cdot \nabla_{z_{n}} \log q_{\phi}(z_{n} | z_{n - 1})$ depends only on $z_{n}, z_{n - 1}$, and the inner integral gives the joint density of $(z_{n},z_{n-1})$. Thus the integral in~\eqref{eq: splitted squared norm term second term integral} can be rewritten as
\begin{align}
    \label{eq: splitted squared norm term second term integral temp}
    & \mathbb{E}_{z_{n}}\left[ g^{2}(t_{n})s_{\theta}(z_{n}, t_{n}) \cdot \nabla_{z_{n}}\log q_{\phi}(z_{n}) \right] = g^{2}(t_{n}) \mathbb{E}_{z_{n - 1}, z_{n}} \left[ s_{\theta}(z_{n}, t_{n}) \cdot \nabla_{z_{n}}\log q_{\phi}(z_{n} | z_{n - 1}) \right].
\end{align}

Since $z_{n} = \mu_{n-1} + \sigma_{n-1}U_{n - 1}$, which is approximately Gaussian when $\Delta t_{n-1}$ is small enough, the score of the transitional distribution $q_{\phi}(z_{n} | z_{n - 1})$ has closed form:

\noindent
\begin{align}
    \label{eq: Markov process transitional distribution score}
    & \nabla_{z_{n}}\log q_{\phi}(z_{n} | z_{n - 1}) = - \frac{U_{n - 1}}{\sigma_{n-1}}.
\end{align}

Then the rewritten integral~\eqref{eq: splitted squared norm term second term integral temp} becomes:

\noindent
\begin{align}
    \label{eq: splitted squared norm term second term integral final}
    & \mathbb{E}_{z_{n}}\left[ g^{2}(t_{n})s_{\theta}(z_{n}, t_{n}) \cdot \nabla_{z_{n}}\log q_{\phi}(z_{n}) \right] = - g^{2}(t_{n}) \mathbb{E}_{z_{n - 1}, z_{n}} \left[ s_{\theta}(z_{n}, t_{n}) \cdot \frac{U_{n-1}}{\sigma_{n-1}} \right].
\end{align}

Note that $\sigma_{n-1} = g(t_{n - 1})\sqrt{\Delta t_{n - 1}}$~\eqref{eq:EM Markov process 2}, where $g$ is the diffusion coefficient of the diffusion process determined by the reference distribution.  When $\Delta t_{n - 1}$ is very small, $\frac{1}{\sigma_{n-1}}$ becomes large, leading to high variance.On the other hand, in order for $q_{\phi}(z_{n} \mid z_{n - 1})$ to be well approximated by a Gaussian, the EM discretization must be fine, so $\Delta t_{n-1}$ must indeed be small. To address this issue, consider the Taylor expansion of the score network:
\begin{align}
    \label{eq: score network Taylor expansion 1}
    & s_{\theta}(z_{n}, t_{n}) = s_{\theta}(\mu_{n-1} + \sigma_{n-1}U_{n-1}, t_{n}) \nonumber \\
    =\; & s_{\theta}(\mu_{n-1}, t_{n}) + \nabla_{\mu_{n-1}}s_{\theta}(\mu_{n-1}, t_{n})\cdot\sigma_{n-1}U_{n-1} + O(\sigma_{n-1}^{2}).
\end{align}
Multipling both sides of~\eqref{eq: score network Taylor expansion 1} by $\frac{U_{n-1}}{\sigma_{n-1}}$ gives
\begin{align}
    \label{eq: score network Taylor expansion 2}
    & s_{\theta}(z_{n}, t_{n}) \cdot \frac{U_{n-1}}{\sigma_{n-1}} = s_{\theta}(\mu_{n-1}, t_{n}) \cdot \frac{U_{n-1}}{\sigma_{n-1}} + \left[ \nabla_{\mu_{n-1}}s_{\theta}(\mu_{n-1}, t_{n}) \cdot U_{n-1} \right] U_{n-1} + O(\sigma_{n-1}).
\end{align}
The first term on the right-hand side, $s_{\theta}(\mu_{n-1}, t_{n}) \cdot \frac{U_{n}}{\sigma_{n-1}}$, becomes large when $\sigma_{n-1}$ is small, but has zero expectation since $U_{n-1}\sim\mathcal{N}(0,I)$. Hence it can be subtracted as a control variate:
\begin{align}
    \label{eq: remove first control variate}
    & \mathbb{E}_{z_{n-1}, z_{n}} \left[ s_{\theta}(\mu_{n-1}, t_{n}) \cdot \frac{U_{n-1}}{\sigma_{n-1}} \right] = \mathbb{E}_{z_{n-1}} \left[ \frac{s_{\theta}(\mu_{n-1}, t_{n})}{\sigma_{n-1}} \mathbb{E}_{z_{n}} \left[ U_{n-1} \right] \right] = 0 \nonumber \\
    \Rightarrow\; & \mathbb{E}_{z_{n - 1}, z_{n}} \left[ s_{\theta}(z_{n}, t_{n}) \cdot \frac{U_{n-1}}{\sigma_{n-1}} \right] \nonumber \\
    & = \mathbb{E}_{z_{n - 1}, z_{n}} \left[ s_{\theta}(z_{n}, t_{n}) \cdot \frac{U_{n-1}}{\sigma_{n-1}} - s_{\theta}(\mu_{n-1}, t_{n}) \cdot \frac{U_{n-1}}{\sigma_{n-1}} \right] \nonumber \\
    & = \mathbb{E}_{z_{n - 1}, z_{n}} \left[ \frac{U_{n-1}}{\sigma_{n-1}} \cdot \left[ s_{\theta}(z_{n}, t_{n}) - s_{\theta}(\mu_{n-1}, t_{n}) \right] \right].
\end{align}
Substituting~\eqref{eq: splitted squared norm term second term integral final} and~\eqref{eq: remove first control variate} into~\eqref{eq: splitted squared norm term} yields
\begin{align}
    \label{eq: splitted squared norm term after first removal}
    & \frac{{T}}{2}\mathbb{E}_{z_{n}}\left[ g^{2}(t_{n}) \rVert \nabla_{z_{n}}\log q_{\phi}(z_{n}) - s_{\theta}(z_{n}, t_{n}) \rVert_{2}^{2} \right] \nonumber \\
    =\; & \frac{{T}}{2}\mathbb{E}_{z_{n}}\left[ g^{2}(t_{n}) \rVert s_{\theta}(z_{n}, t_{n}) \rVert_{2}^{2} \right] \nonumber \\
    & + {T}g^{2}(t_{n})\mathbb{E}_{z_{n - 1}, z_{n}} \left[ \frac{U_{n-1}}{\sigma_{n-1}} \cdot \left[ s_{\theta}(z_{n}, t_{n}) - s_{\theta}(\mu_{n-1}, t_{n}) \right] \right] \nonumber \\
    &+ \frac{{T}}{2}\mathbb{E}_{z_{n}}\left[ g^{2}(t_{n}) \rVert \nabla_{z_{n}}\log q_{\phi}(z_{n}) \rVert_{2}^{2} \right].
\end{align}

For simplicity, the above discussion is based on $N$ takes some value $n$. Reintroducing the random index $N$ in place of $n$ gives
\begin{align}
    \label{eq: splitted squared norm term after first removal N back}
    & \frac{{T}}{2}\mathbb{E}_{N, z_{N}}\left[ g^{2}(t_{N}) \rVert \nabla_{z_{N}}\log q_{\phi}(z_{N}) - s_{\theta}(z_{N}, t_{N}) \rVert_{2}^{2} \right] \nonumber \\
    =\; & \frac{{T}}{2}\mathbb{E}_{N, z_{N}}\left[ g^{2}(t_{N}) \rVert s_{\theta}(z_{N}, t_{N}) \rVert_{2}^{2} \right] \nonumber \\
    & + {T}\mathbb{E}_{N} \left[ g^{2}(t_{N})\mathbb{E}_{z_{N - 1}, z_{N}} \left[ \frac{U_{N-1}}{\sigma_{N-1}} \cdot \left[ s_{\theta}(z_{N}, t_{N}) - s_{\theta}(\mu_{N-1}, t_{N}) \right] \right] \right] \nonumber \\
    &+ \frac{{T}}{2}\mathbb{E}_{N, z_{N}}\left[ g^{2}(t_{N}) \rVert \nabla_{z_{N}}\log q_{\phi}(z_{N}) \rVert_{2}^{2} \right].
\end{align}

The third squared norm term in~\eqref{eq: splitted squared norm term after first removal N back} can be cancelled out by part of the third inner product term in~\eqref{eq: cross entropy 3.2}. So now the cross entropy term in the loss function can be approximately bounded by:

\noindent
\begin{align}
    \label{eq: cross entropy 3.3}
    & \text{CE}(q_{\phi}(z_{0}|x) \rVert p_{\theta}(z_{0})) \nonumber \\
    \lesssim\; & H(q_{\phi}(z_{n_{f}})) \nonumber \\
    & + \frac{{T}}{2}\mathbb{E}_{N, z_{N}}\left[ g^{2}(t_{N}) \rVert s_{\theta}(z_{N}, t_{N}) \rVert_{2}^{2} \right] + {T}\mathbb{E}_{N} \left[ g^{2}(t_{N})\mathbb{E}_{z_{N - 1}, z_{N}} \left[ \frac{U_{N-1}}{\sigma_{N-1}} \cdot \left[ s_{\theta}(z_{N}, t_{N}) - s_{\theta}(\mu_{N-1}, t_{N}) \right] \right] \right] \nonumber \\
    & + {T}\mathbb{E}_{N, z_{N}}\left[ f(z_{N}, t_{N}) \cdot \nabla_{z_{N}}\log q_{\phi}(z_{N}) \right].
\end{align}

For the third inner product term in~\eqref{eq: cross entropy 3.3}, we can also write it in the form of integral as~\eqref{eq: splitted squared norm term second term integral} to utilize the approximately Gaussian transition distribution. Also consider $N$ takes some value $n$ for simplicity:

\noindent
\begin{align}
    \label{eq: second integral formulation}
    & \mathbb{E}_{z_{n}} \left[ f(z_{n}, t_{n}) \cdot \nabla_{z_{n}}\log q_{\phi}(z_{n}) \right] \nonumber \\
    =\; & \int f(z_{n}, t_{n}) \cdot q_{\phi}(z_{n}) \nabla_{z_{n}}\log q_{\phi}(z_{n}) \dd z_{n} \nonumber \\
    =\; & \int f(z_{n}, t_{n}) \cdot \nabla_{z_{n}}q_{\phi}(z_{n}) \dd z_{n} \nonumber \\
    =\; & \int f(z_{n}, t_{n}) \cdot \nabla_{z_{n}} \left[ \int \cdots \int \int q_{\phi}(z_{n} | z_{n - 1}) \cdots q_{\phi}(z_{1} | z_{0}) q_{\phi}(z_{0}) \dd z_{0} \dd z_{1} \cdots \dd z_{n - 1} \right] \dd z_{n} \nonumber \\
    =\; & \int \cdots \int \int \int \left[ f(z_{n}, t_{n}) \cdot \nabla_{z_{n}} \log q_{\phi}(z_{n} | z_{n - 1}) \right] q_{\phi}(z_{n} | z_{n - 1}) \cdots q_{\phi}(z_{1} | z_{0}) q_{\phi}(z_{0}) \dd z_{0} \dd z_{1} \cdots \dd z_{n - 1} \dd z_{n}.
\end{align}

Similar to~\eqref{eq: splitted squared norm term second term integral}, based on the Markov property~\eqref{eq: Markov property}, we interpret~\eqref{eq: second integral formulation} as
\begin{align}
    \label{eq: second control variate 1}
    & \mathbb{E}_{z_{n}}\left[ f(z_{n}, t_{n}) \cdot \nabla_{z_{n}}\log q_{\phi}(z_{n}) \right] =\mathbb{E}_{z_{n-1}, z_{n}}\left[ f(z_{n}, t_{n}) \cdot \nabla_{z_{n}}\log q_{\phi}(z_{n} | z_{n-1}) \right].
\end{align}
Using again that $q_{\phi}(z_{n} \mid z_{n-1})$ is approximately Gaussian, we obtain
\begin{align}
    \label{eq: second control variate 2}
    & \mathbb{E}_{z_{n}}\left[ f(z_{n}, t_{n}) \cdot \nabla_{z_{n}}\log q_{\phi}(z_{n}) \right] = -\mathbb{E}_{z_{n-1}, z_{n}}\left[ f(z_{n}, t_{n}) \cdot \frac{U_{n-1}}{\sigma_{n-1}} \right].
\end{align}
This term also suffers from high variance when $\Delta t_{n-1}$ is small. As before, we consider the Taylor expansion
\begin{align}
    \label{eq: second control variate 3, Taylor expansion}
    & f(z_{n}, t_{n}) = f(\mu_{n-1}, t_{n}) + \nabla_{\mu_{n-1}} f(\mu_{n-1}, t_{n}) \cdot \sigma_{n-1} U_{n-1} + O(\sigma_{n-1}^{2}).
\end{align}
Multiplying~\eqref{eq: second control variate 3, Taylor expansion} by $\frac{U_{n-1}}{\sigma_{n-1}}$ shows that $f(\mu_{n-1}, t_{n}) \cdot \frac{U_{n-1}}{\sigma_{n-1}}$ is the dominant exploding term. Its expectation is zero since $U_{n-1}$ is centered Gaussian, so it can also be removed as a control variate:
\noindent
\begin{align}
    \label{eq: second control variate 3}
    & \mathbb{E}_{z_{n-1}, z_{n}} \left[ f(\mu_{n-1}, t_{n}) \cdot \frac{U_{n-1}}{\sigma_{n-1}} \right] = \mathbb{E}_{z_{n-1}} \left[ \frac{f(\mu_{n-1}, t_{n})}{\sigma_{n-1}} \mathbb{E}_{z_{n}} \left[  U_{n-1} \right] \right] = 0 \nonumber \\
    \Rightarrow\; & \mathbb{E}_{z_{n-1}, z_{n}}\left[ f(z_{n}, t_{n}) \cdot \frac{U_{n-1}}{\sigma_{n-1}} \right] = \mathbb{E}_{z_{n-1}, z_{n}}\left[ f(z_{n}, t_{n}) \cdot \frac{U_{n-1}}{\sigma_{n-1}} - f(\mu_{n-1}, t_{n}) \cdot \frac{U_{n-1}}{\sigma_{n-1}} \right] \nonumber \\
    & = \mathbb{E}_{z_{n-1}, z_{n}} \left[ \frac{U_{n-1}}{\sigma_{n-1}} \left[ f(z_{n}, t_{n}) - f(\mu_{n-1}, t_{n}) \right] \right].
\end{align}
Reinstating $N$ in place of $n$ as a random variable and substituting into~\eqref{eq: cross entropy 3.3} gives
\begin{align}
    \label{eq: cross entropy 3.4}
    & \text{CE}(q_{\phi}(z_{0}|x) \rVert p_{\theta}(z_{0})) \nonumber \\
    \lesssim\; & H(q_{\phi}(z_{n_{f}})) + \frac{{T}}{2}\mathbb{E}_{N, z_{N}}\left[ g^{2}(t_{N}) \rVert s_{\theta}(z_{N}, t_{N}) \rVert_{2}^{2} \right] \nonumber \\
    & + {T}\mathbb{E}_{N} \left[ g^{2}(t_{N})\mathbb{E}_{z_{N - 1}, z_{N}} \left[ \frac{U_{N-1}}{\sigma_{N-1}} \cdot \left[ s_{\theta}(z_{N}, t_{N}) - s_{\theta}(\mu_{N-1}, t_{N}) \right] \right] \right] \nonumber \\
    & - {T}\mathbb{E}_{N} \left[ \mathbb{E}_{z_{N-1}, z_{N}} \left[ \frac{U_{N-1}}{\sigma_{N-1}} \left[ f(z_{N}, t_{N}) - f(\mu_{N-1}, t_{N}) \right] \right] \right],
\end{align}
where the random variable $N$ is uniformly distributed on $\{1, \cdots, n_{f}\}$. This concludes the proof.
\end{proof}


\section{NVAE Architecture}\label{section:appendix C}

The NVAE architecture in our experiments is the simplified version of the one implemented in \cite{LSGM}. Here we illustrate the simplified architecture used for the MNIST dataset. We still set the number of latent scales to 1. We set the number of groups of latent variables in this only latent scale to 2. We change the number of channels of the latent variables in each group from 20 to 1. We change the number of channels of the output feature representation of the initial convolution layer from 64 to 16. Such changes lead to a much simpler network, with the number of trainable parameters drastically reduced from 3,005,822 to 207,648.

\subsection{Input}

A batch of samples from the MNIST dataset is first zero-padded to obtain a $B \times 1 \times 32 \times 32$ feature, where $B$ is the batch size. This feature is then passed to a $3 \times 3$ convolutional layer with padding of 1 and bias, which outputs a $B \times 16 \times 32 \times 32$ feature.

\subsection{Preprocess Block}

The preprocess block contains 6 cells. Each cell is based on the residual cell introduced in \cite{NVAE}. The input is passed to a batch normalization (BN) layer, activated by the Swish activation function \cite{swish} and passed to a $3 \times 3$ convolutional layer. The output goes through the same process again and is passed to a squeeze and excitation (SE) layer \cite{SE}. The first convolutional layer in the third cell has stride of 2 and outputs a $B \times 32 \times 16 \times 16$ feature. The first convolutional layer in the final cell has stride of 2 and outputs a $B \times 32 \times 8 \times 8$ feature. All six cells extend the residual connection to skip connection with coefficient 0.1.

\subsection{Encoder Tower}

The encoder tower contains 2 cells. Each cell has the same architecture as the cells in the preprocess block with stride of 1. The output of the first cell is a $B \times 32 \times 8 \times 8$ feature and is passed to both the second cell and the encoder combiner in the decoder tower. The output of the second cell is also a $B \times 32 \times 8 \times 8$ feature and is passed to the normal sampler.

\subsection{Normal Sampler}

The purpose of the normal sampler is to sample from the normal distribution. The output of the second cell in the encoder tower is activated by the ELU activation function \cite{elu} and passed to a convolutional layer. The output is a $B \times 2 \times 8 \times 8$ feature, which is splitted into two $B \times 1 \times 8 \times 8$ features and used as the mean and covariance of a normal distribution. A sample is generated from this distribution and passed to a convolutional layer to obtain a $B \times 1 \times 8 \times 8$ feature. This feature is then passed to the first decoder combiner in the decoder tower.

The whole process above is repeated for the output of the encoder combiner in the decoder tower. The output feature is then passed to the second decoder combiner in the decoder tower. There is also a $B \times 1 \times 8 \times 8$ sample from a normal distribution. This sample and the sample in the process above are concatenated as a $B \times 2 \times 8 \times 8$ feature, which is viewed as a batch of latent variables in the latent space determined by the encoder of this NVAE.

\subsection{Decoder Tower}

The decoder tower contains 2 cells. Each cell is modified from the residual cell. The BN $\rightarrow$ Swish activation $\rightarrow$ convolution is replaced by the convolution $\rightarrow$ BN $\rightarrow$ Swish activation. An extra BN is added at the beginning. An extra convolution $\rightarrow$ BN is added before the final SE layer. The decoder tower also contains one encoder combiner and two decoder combiners. The first decoder combiner takes the output of the normal sampler and a $32 \times 8 \times 8$ learnable parameter as inputs. It outputs a $B \times 32 \times 8 \times 8$ feature, which is passed to the first cell in the decoder tower. The encoder combiner takes the outputs of the first cells in both the encoder tower and the decoder tower as inputs. The output is passed to the normal sampler. The normal sampler's output and the first cell's output are then passed to the second decoder combiner. Finally, the decoder combiner's output is passed to the second cell. The output $B \times 32 \times 8 \times 8$ feature is passed to the postprocess block.

\subsection{Postprocess Block}

The postprocess block contains 6 cells. Each cell has the same architecture as the cells in the decoder tower. The first convolutional layer in the first cell has stride of -1 and outputs a $B \times 32 \times 16 \times 16$ feature. The first convolutional layer in the third cell has stride of -1 and outputs a $B \times 16 \times 32 \times 32$ feature. This output is obtained by skip connection, the same as the preprocess block.

\subsection{Output}

The output of the postprocess block is activated by the ELU function and passed to a $3 \times 3$ convolution layer. The output is a $B \times 1 \times 32 \times 32$ feature, which can be viewed as logits. The Bernoulli distribution determined by the logits can be viewed as the output of the decoder of this NVAE.


\section{Experiment Results}\label{section:appendix D}

\begin{figure}[htbp]
  \centering
  
  \begin{subfigure}[t]{0.24\linewidth}
    \centering
    \includegraphics[width=\linewidth]{figures/GMM_weights/MNIST-samples.png}
    \caption{Training samples}
    \label{fig:samples and frequency 3000 a}
  \end{subfigure}
  \hspace{0.5em}
  \begin{subfigure}[t]{0.24\linewidth}
    \centering
    \includegraphics[width=\linewidth]{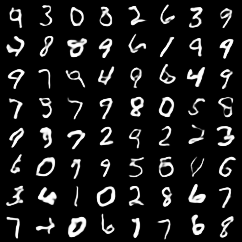}
    \caption{LSGM}
    \label{fig:samples and frequency 3000 b}
  \end{subfigure}
  \hspace{0.5em}
  \begin{subfigure}[t]{0.24\linewidth}
    \centering
    \includegraphics[width=\linewidth]{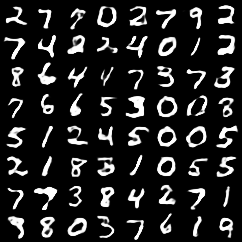}
    \caption{\textbf{LNDSM-SGM}}
    \label{fig:samples and frequency 3000 c}
  \end{subfigure}

  \vspace{0.5em}
  
  \begin{subfigure}[t]{0.24\linewidth}
    \centering
    \includegraphics[width=\linewidth]{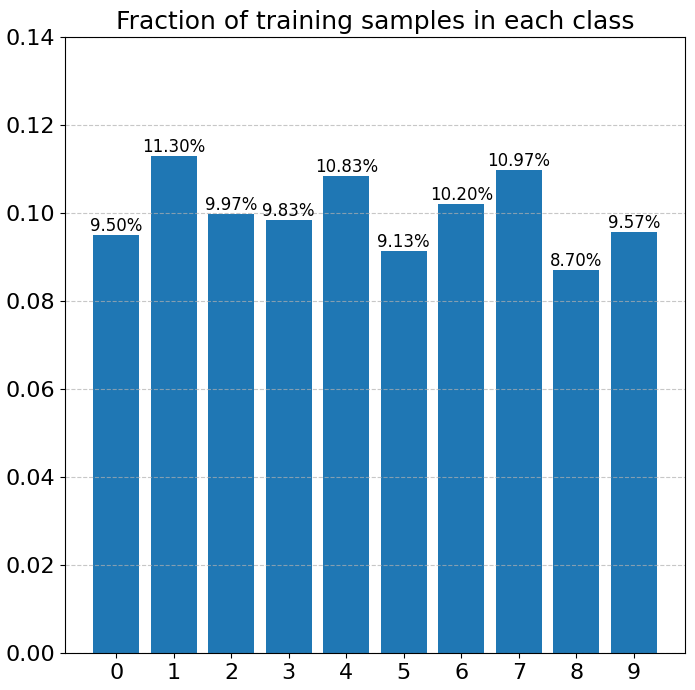}
    \caption{Training samples}
    \label{fig:samples and frequency 3000 d}
  \end{subfigure}
  \hspace{0.5em}
  \begin{subfigure}[t]{0.24\linewidth}
    \centering
    \includegraphics[width=\linewidth]{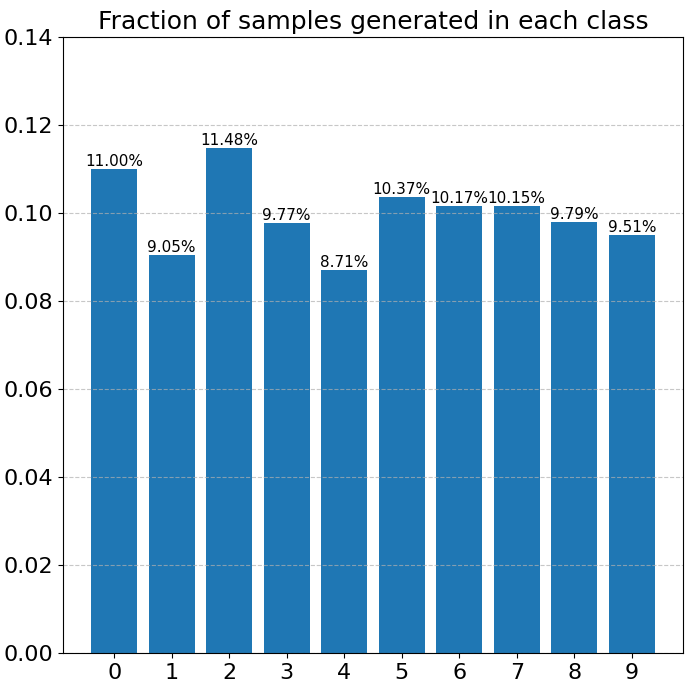}
    \caption{LSGM}
    \label{fig:samples and frequency 3000 e}
  \end{subfigure}
  \hspace{0.5em}
  \begin{subfigure}[t]{0.24\linewidth}
    \centering
    \includegraphics[width=\linewidth]{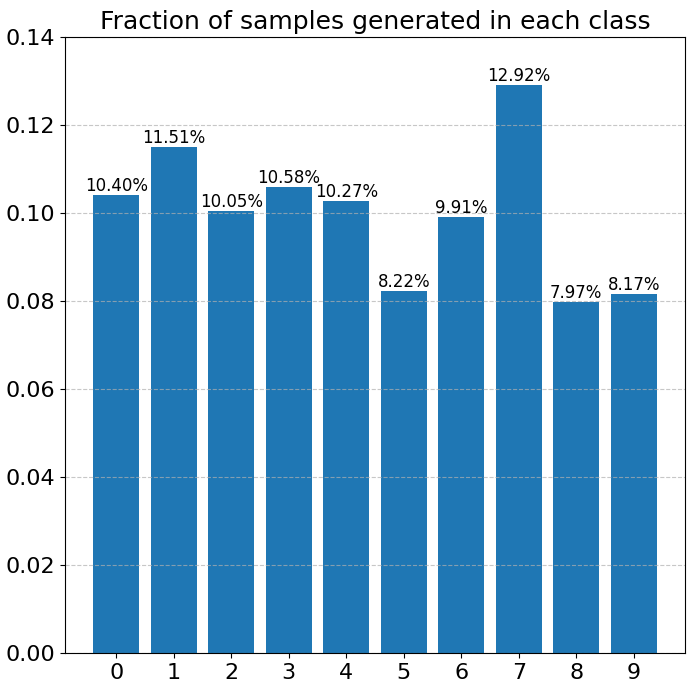}
    \caption{\textbf{LNDSM-SGM}}
    \label{fig:samples and frequency 3000 f}
  \end{subfigure}
  
  \caption{LSGM and LNDSM-SGM trained using the low data MNIST dataset. (a) and (d): Snapshot and fraction of different digits of the training samples. (b) and (e): Snapshot and fraction of different digits of the 10,000 samples generated by the LSGM. The KL divergence from the fraction of the training samples to the fraction of the generated samples: 0.00883. (c) and (f): Snapshot and fraction of different digits of the 10,000 samples generated by the LNDSM-SGM. The KL divergence from the fraction of the training samples to the fraction of the generated samples: 0.00439.}
  \label{fig:samples and frequency 3000}
\end{figure}

\begin{table}[htbp]
    \centering
    \caption{Low data MNIST ($N=3000$)}
    \label{tab:FID IS low data}
    \begin{tabular}{cccc}
    \toprule
    model & FID ($\downarrow$) & IS ($\uparrow$) & KL ($\downarrow$) \\
    \midrule
        LSGM          &   27.1               &      9.12           &   0.00883   \\
        LNDSM-SGM     &   \textbf{15.0}      &      \textbf{9.57}  &   \textbf{0.00439}   \\
    \bottomrule
    \end{tabular}
\end{table}

\begin{figure}[htbp]
  \centering
  
  \begin{subfigure}[t]{0.24\linewidth}
    \centering
    \includegraphics[width=\linewidth]{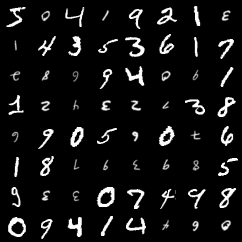}
    \caption{Training samples}
    \label{fig:samples and frequency approx-C2 a}
  \end{subfigure}
  \hspace{0.5em}
  \begin{subfigure}[t]{0.24\linewidth}
    \centering
    \includegraphics[width=\linewidth]{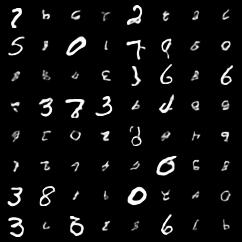}
    \caption{LSGM}
    \label{fig:samples and frequency approx-C2 b}
  \end{subfigure}
  \hspace{0.5em}
  \begin{subfigure}[t]{0.24\linewidth}
    \centering
    \includegraphics[width=\linewidth]{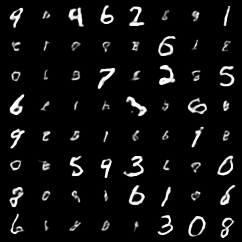}
    \caption{\textbf{LNDSM-SGM}}
    \label{fig:samples and frequency approx-C2 c}
  \end{subfigure}

  \vspace{0.5em}
  
  \begin{subfigure}[t]{0.24\linewidth}
    \centering
    \includegraphics[width=\linewidth]{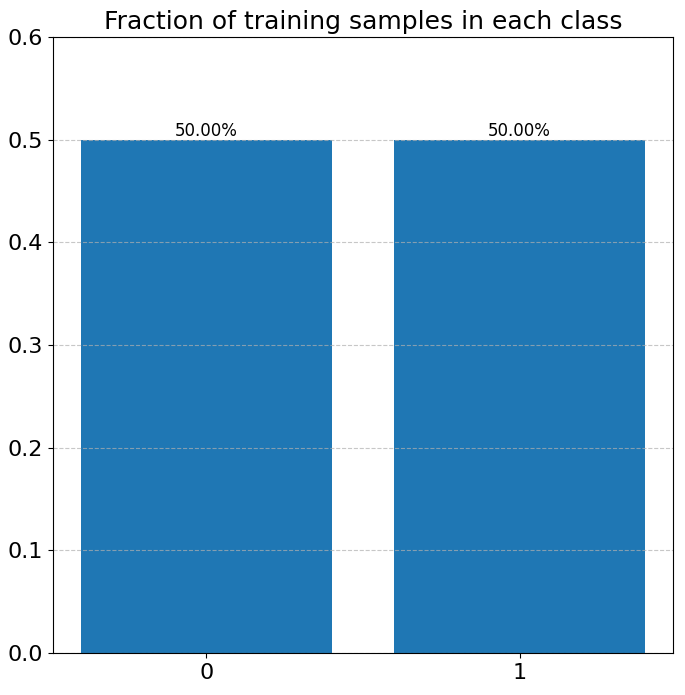}
    \caption{Training samples}
    \label{fig:samples and frequency approx-C2 d}
  \end{subfigure}
  \hspace{0.5em}
  \begin{subfigure}[t]{0.24\linewidth}
    \centering
    \includegraphics[width=\linewidth]{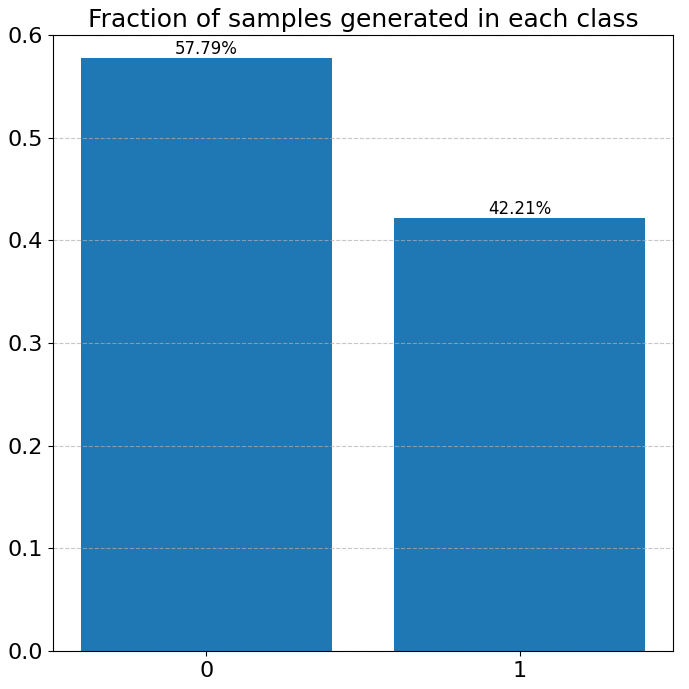}
    \caption{LSGM}
    \label{fig:samples and frequency approx-C2 e}
  \end{subfigure}
  \hspace{0.5em}
  \begin{subfigure}[t]{0.24\linewidth}
    \centering
    \includegraphics[width=\linewidth]{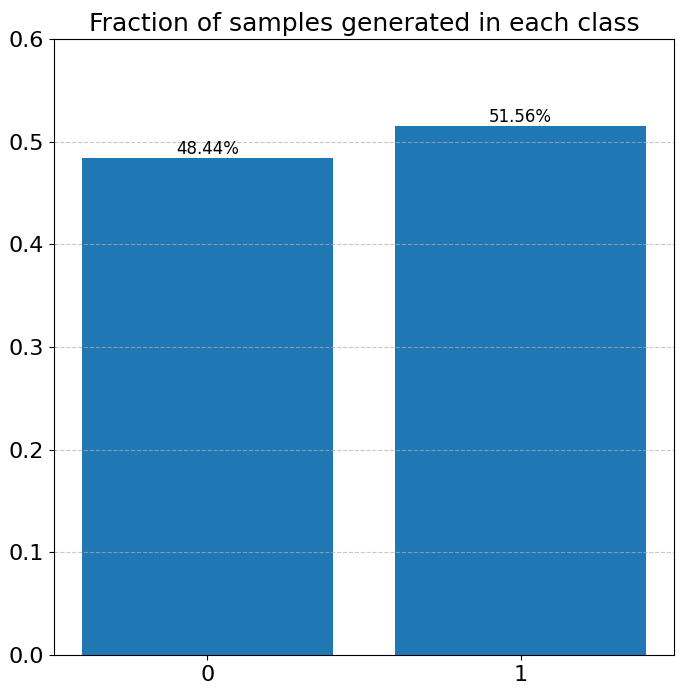}
    \caption{\textbf{LNDSM-SGM}}
    \label{fig:samples and frequency approx-C2 f}
  \end{subfigure}
  
  \caption{LSGM and LNDSM-SGM trained using the full Approx.-C2-MNIST dataset. (a) and (d): Snapshot and fraction of different digits of the training samples. (b) and (e): Snapshot and fraction of different digits of the 10,000 samples generated by the LSGM. (c) and (f): Snapshot and fraction of different digits of the 10,000 samples generated by the LNDSM-SGM.}
  \label{fig:samples and frequency approx-C2}
\end{figure}


\end{document}